\documentclass[11pt]{article}
\usepackage{fullpage}

\usepackage[numbers]{natbib}

\usepackage[T1]{fontenc}
\usepackage{lmodern}
\usepackage[utf8]{inputenc} 
\usepackage{hyperref}       
\usepackage{url}            
\usepackage{booktabs}       
\usepackage{nicefrac}       
\usepackage{microtype}      
\usepackage[dvipsnames]{xcolor}         

\usepackage{amsmath, amsthm, amssymb, amsfonts}
\usepackage{bm}
\usepackage{mathtools}
\usepackage{hyperref}
\usepackage[capitalise]{cleveref}
\usepackage{algorithm}  
\usepackage[noend]{algorithmic}

\usepackage{tabularx}
\usepackage{thm-restate}
\usepackage{subfig}
\usepackage{color}
\usepackage{enumitem}

\usepackage{mathtools}
\usepackage{thmtools}

\usepackage{amssymb} 
\usepackage{multicol} 
\usepackage{tabularx}       
\usepackage{minitoc}
\mtcsettitle{parttoc}{ }    

\declaretheorem[name=Theorem,numberwithin=section]{thm}
\declaretheorem[name=Lemma,numberwithin=section]{lem}
\declaretheorem[name=Proposition,numberwithin=section]{prop}
\declaretheorem[name=Remark,numberwithin=section]{rem}
\declaretheorem[name=Corollary,numberwithin=section]{cor}
\theoremstyle{definition}
\declaretheorem[name=Definition,numberwithin=section]{defn}
\declaretheorem[name=Example,numberwithin=section]{exmp}
\newtheorem{assumption}{Assumption}

\newtheorem{assumptionalt}{Assumption}[assumption]
\newenvironment{assumptionp}[1]{
  \renewcommand\theassumptionalt{#1}
  \assumptionalt
}{\endassumptionalt}

\newcommand{\N}{{\mathbb N}}
\renewcommand{\P}{{\mathbf P}}

\newcommand{\E}{{\mathbf E}}

\newcommand{\1}{{\mathbf 1}}

\newcommand{\cA}{{\mathcal A}}

\newcommand{\cE}{{\mathcal E}}

\newcommand{\cF}{{\mathcal F}}

\newcommand{\bP}{{\mathbf P}}
\newcommand{\tS}{{\tilde {S}}}

\newcommand{\tO}{\tilde{O}}

\newcommand{\nstdb}{{\texttt{NSt-DB}}}

\newcommand{\bv}{\textrm{Benchmark-Switches}}
\newcommand{\mv}{\textrm{Total Variation}}
\newcommand{\cwmv}{\textrm{Condorcet Winner Variation}}
\newcommand{\cv}{\textrm{Condorcet Winner Variation}}
\newcommand{\cwv}{\textrm{Condorcet Winner Switches}}
\newcommand{\pv}{\textrm{Preference Switches}}
\newcommand{\sigv}{\textrm{Significant Condorcet Winner Switches}}

\renewcommand{\a}{{a}}
\renewcommand{\b}{{b}}

\newcommand{\sm}{\setminus}

\def \papertitle{{\fontsize{16}{11}\selectfont ANACONDA: An Improved Dynamic Regret Algorithm \\ for Adaptive Non-Stationary Dueling Bandits}}

\newcommand{\red}[1]{\textcolor{red}{#1}}

\makeatletter
\newcommand{\multiline}[1]{%
  \begin{tabularx}{\dimexpr\linewidth-\ALG@thistlm}[t]{@{}X@{}}
    #1
  \end{tabularx}
}
\makeatother

\newcommand{\abs}[1]{\left\vert #1 \right\vert}
\DeclarePairedDelimiter\floor{\lfloor}{\rfloor}

\newcommand{\mainalgo}{\texttt{ANACONDA}}
\newcommand{\phases}{\textrm{Phases}}
\newcommand{\hdelta}{\hat \delta}
\newcommand{\mA}{\cA_{\textrm{good}}}
\newcommand{\curA}{\cA_{\textrm{local}}}
\renewcommand{\E}{\mathbb{E}}   
\renewcommand{\P}{\mathbb{P}}   
\newcommand{\terma}{R_1(\ell)}
\newcommand{\termb}{R_2(\ell)}
\newcommand{\termone}{\tilde R_1(\ell)}   
\newcommand{\termtwo}{\tilde R_2(\ell)}   

\newcommand{\base}{{\texttt{CondaLet}}}

\newcommand{\Sw}{S^{\texttt{CW}}}     
\newcommand{\sigS}{\tS^{\texttt{CW}}} 
\newcommand{\Sp}{S^{\texttt{P}}}    
\newcommand{\Vt}{V} 
\newcommand{\Vcw}{{\tilde V}} 
\newcommand{\bern}{\textrm{Bern}}
\newcommand{\tstart}{t_0} 
\newcommand{\mstart}{m_0}
\newcommand{\badseg}{\textrm{bad}}
\newcommand{\const}{\tilde c}  
\newcommand{\htau}{\nu} 
\newcommand{\agood}{a^{g}_\ell}
\newcommand{\asafe}{a^{s}}
\newcommand{\DR}{\textup{\textrm{DR}}}
\newcommand{\SR}{\textup{\textrm{R}}}

\usepackage{tikz}   
\usetikzlibrary{patterns,arrows,decorations.pathreplacing}

\usepackage{stmaryrd}

\usepackage{array}
\usepackage{eqparbox}

\usepackage{enumitem}
\makeatletter
\newcommand{\mylabel}[2]{#2\def\@currentlabel{#2}\label{#1}}
\makeatother

\title{\bfseries \papertitle \vspace{0.4cm}}

\author{
Thomas Kleine Buening%
\thanks{University of Oslo, Norway; {\tt thomkl@ifi.uio.no}.}
\and 
Aadirupa Saha%
\thanks{Toyota Technological Institute at Chicago (TTIC), USA; {\tt aadirupa@ttic.edu}.}
}
\date{\vspace{-0.4cm}}

\begin{document}

\maketitle

\begin{abstract}

We study the problem of non-stationary dueling bandits and provide the first adaptive dynamic regret algorithm for this problem. 
The only two existing attempts in this line of work fall short across multiple dimensions, including pessimistic measures of non-stationary complexity and non-adaptive parameter tuning that requires knowledge of the number of preference changes.
We develop an elimination-based rescheduling algorithm to overcome these shortcomings and show a near-optimal $\tO(\sqrt{\Sw T})$ dynamic regret bound, where $\Sw$ is the number of times the Condorcet winner changes in $T$ rounds. This yields the first near-optimal dynamic regret algorithm for unknown $\Sw$. 
We further study other related notions of non-stationarity for which we also prove near-optimal dynamic regret guarantees under 
additional assumptions on the underlying preference model. 
\end{abstract}

\section{Introduction}
\label{sec:intro}

Multi-Armed Bandits (MAB) \citep{thompson1933likelihood, robbins1952some, CsabaNotes18} are a well-studied online learning framework, which can be used to model online decision-making under uncertainty. Due to its exploration-exploitation tradeoff, the MAB framework is able to model situations such as clinical trials or job scheduling, where the goal is to keep selecting the `best item' in hindsight through sequentially querying one item at a time and subsequently observing a noisy reward feedback for the queried item \citep{Auer+02,Audibert+10,TS12,BubeckNotes+12}.

The MAB framework has been studied and generalized to different settings, among which a popular variant is known as {Dueling Bandits} (DB) which has gained much attention in the machine learning community over the last two decades \citep{Yue+12,Zoghi+14RCS,Zoghi+15, DTS}. DB are a preference-based variant of MAB in which every round the learner selects a pair of items (or arms) whereupon a noisy preference between the two items is observed. Such a model is particularly useful in applications, where direct numerical feedback is unavailable, but observed feedback or behavior implies a preference of one item over the other. For instance, the DB framework can be used for search engine optimization through interleaved comparisons~\citep{radlinski2013optimized, hofmann2011probabilistic}. 

%

In the classical stochastic dueling bandit problem, it is assumed that the underlying preferences between items remain fixed over time. However, this assumed stationarity of preferences is likely to be violated in many applications. For example, preferences over movies may change depending on the season or other external influences. 
Despite its strong practical motivation, regret minimization in non-stationary dueling bandits has only recently been studied for the fist time~\citep{SahaNDB, BengsNDB}.  
In contrast to the classical stochastic \citep{Yue+12, Zoghi+14RUCB, Busa21survey} and adversarial \cite{Adv_DB, ADB, VDB} dueling bandit problem, which measures performance in terms of static regret w.r.t.\ a fixed benchmark (or best item in hindsight), in non-stationary dueling bandits we consider the stronger \emph{dynamic regret}, which compares the algorithm's selection against a dynamic benchmark every round. 

In general, the achievable dynamic regret depends on the amount of non-stationarity in the environment. Here, prior work \cite{SahaNDB, BengsNDB} studied the number of changes in the preference matrix as a measure non-stationary complexity. 
While the number of such preference switches indeed relates to the hardness of the problem, it is, however, a pessimistic measure of non-stationarity. For example, a change in the preference between two widely suboptimal arms or a minor change in the preference matrix under which the optimal arm remains optimal should not significantly impact our ability to achieve low dynamic regret. To this end, one question that we aim to address in this paper for the paradigm of non-stationary dueling bandits is:
\begin{quote}
    \textbf{Q.1:} \emph{Can we guarantee low dynamic regret for stronger and more meaningful notions of non-stationarity?}
\end{quote} 
Moreover, prior work in non-stationary dueling bandits~\citep{SahaNDB, BengsNDB} assumes knowledge of the non-stationary complexity, i.e.\ prior knowledge of the total number of preference switches (or total variation), which is a highly impractical assumption. The second question we thus address is:

\begin{quote}
    \textbf{Q.2:} \emph{Can we achieve near-optimal dynamic regret in non-stationary dueling bandits adaptively, without the knowledge of the underlying non-stationary complexity?} 
\end{quote}

\subsection{Our Contributions} 
We answer these two questions affirmatively. Our main contribution is a new algorithm $\mainalgo$ that adaptively achieves near-optimal regret with respect to the number of 'best arm' switches---a measure that is sensitive only to the variations of the best arms in the preference sequence and indifferent to any other `background noise' due to suboptimal arms. More precisely, our contributions can be listed as follows:

%




\begin{itemize}[leftmargin=0pt, itemindent=10pt, labelwidth=5pt, labelsep=5pt, topsep=0pt,itemsep=2pt]
\item \textbf{Connecting Different Notions of Non-Stationary Complexity in DB.} We first give an overview over different notions of non-stationarity measures for dueling bandits and analyze their interdependencies towards a better understanding of the implications of one to another (\cref{sec:nst_measures}).

\item \textbf{Proposing Tighter Notions of Non-Stationarity (towards Q.1). } We propose three new notions of non-stationary complexity for dueling bandits: (i)~$\Sw$ which measures the number of Condorcet Winner Switches in the preference sequence, (ii)~$\Vcw$ which measures the preference variation of the Condorcet arms, and (iii)~$\sigS$ that counts only the `significant variations' in the Condorcet arms 
(\cref{sec:nst_measures}). 

The novelty of our proposed non-stationarity measures lies in capturing only the non-stationarity observed for the `best arms' of the preference sequences. They remain unaffected by any changes in the suboptimal arms, which of course captures a stronger notion of non-stationarity than simply counting the number of preference shifts $\Sp$, or total variation $\Vt$, of the preference sequence $\{\bP_t\}_{t \in [T]}$, as studied in prior work \cite{SahaNDB,BengsNDB}. In particular, we show that $\sigS \leq \Sw \leq \Sp$ and $\Vcw \leq \Vt$ justifying the strength of our proposed non-stationarity measures. 

\item \textbf{Adaptive Algorithm (towards Q.2).} 
Besides using weaker notions of non-stationary complexity, another drawback of existing work on non-stationary dueling bandit is that, in order to optimize dynamic regret, their algorithms require exact knowledge of the non-stationary complexity (e.g.\ $\Sp$ or $\Vt$), which is in practice of course expected to be unknown to the system/algorithm designed ahead of time. 
Our next main contribution lies in designing an adaptive algorithm ($\mainalgo$, \cref{alg:anaconda}) that does not require knowledge of any underlying non-stationary complexity---it can adapt to any unknown number of best arm switches $\Sw$ and yields a near-optimal regret bound of $\tO\big(\sqrt{\Sw T}\big)$ (\cref{thm:main_result}, Section~\ref{sec:algo}).\footnote{Here, $\tO$ notation hides logarithmic dependencies.}

\item \textbf{Improved and (Near-)Optimal Dynamic Regret Bounds.} Owing to the fact that $\Sw \leq \Sp$, our dynamic regret bounds can be much tighter compared to the previous results by \cite{SahaNDB,BengsNDB} which can only give a regret guarantee of $\tO\big(\sqrt{\Sp T}\big)$ (\cref{rem:cwv_vs_pv}). Further our regret bound is also provably order optimal in $T$ and $\Sw$ as justified in \cref{rem:lb}.    

\item \textbf{Better Guarantees for Structured Preferences.} Moreover, in \cref{sec:sign_switches} we discover a special class of preference matrices, those that respect a type of transitive property, for which we can prove even stronger dynamic regret guarantees of ${\tO\big(\sqrt{\sigS T}\big)}$ in terms of Significant CW Switches $\sigS$ and $\smash{\tO\big({{\Vcw}^{\nicefrac{1}{3}} T^{\nicefrac{2}{3}}}\big)}$ in terms of $\cwmv$ $\Vcw$. The optimality of these bounds is discussed in \cref{rem:rush} and \cref{rem:tv_under_STT_STI}. 

\end{itemize}

\subsection{Related Works}
\label{sec:rel}

The non-stationary MAB problem has been extensively studied for various non-stationarity measures, such as total variation~\citep{besbes+14, besbes+15}, distribution switches~\citep{garivier2011upper, allesiardo2017non, auer+19}, or best arm switches~\citep{Abasi2022A, Suk22}. Moreover, its study has been extended to more complex setups including linear bandits~\citep{cappe1, cappe2} and contextual MAB~\citep{luo+18, luo+19, wu+18}. We will particularly take inspiration from the recent advances of \citep{auer+19, Abasi2022A, Suk22} that were able to achieve near-optimal dynamic regret rates without knowledge of the number of distribution (or best arm) changes. 

While the non-stationary MAB problem has seen much attention in recent years, its DB counterpart remains widely unexplored. The only two earlier works that address the non-stationary dueling bandit problem are \cite{SahaNDB} and \cite{BengsNDB}. However, these works are limited in a) the weakness of the analyzed non-stationarity measures, namely, general preference switches or total variation (see Section~\ref{sec:nst_measures}), and b) in the fact that their algorithms require knowledge of the total amount of non-stationarity in advance, an unrealistic assumption. 
Here, we improve upon prior work by designing an adaptive algorithm $\mainalgo$ that does not require knowledge of the amount of non-stationarity in the environment and achieves near-optimal dynamic regret w.r.t.\ the number of Condorcet winner switches, a stronger notion of non-stationarity than general preference switches. A more detailed review of previous work that is related to the non-stationary MAB and DB problem is provided in Appendix~\ref{app:rel_works}.

\section{Problem Setting}\label{sec:prob}




We consider preference matrices $\bP \in [0,1]^{K \times K}$ such that $P(\a, \b)$ indicates the probability of arm $\a$ being preferred over arm $\b$. Here, $\bP$ satisfies $P(\a, \b) = 1- P(\a, \b)$ and $P(\a, \a) = 0.5$ for all $\a, \b \in [K]$. We say that $\a$ dominates $\b$ and write $\a \succ \b$ if $P(\a, \b) > 0.5$, i.e.\ arm $\a$ has a higher chance of winning over arm $\b$ in a duel $(\a, \b)$. 
A well-studied concept of a \emph{good benchmark arm} in dueling bandits is the \emph{Condorcet Winner} (CW): Given any preference matrix $\bP \in [0,1]^{K \times K}$, an arm $\a^* \in [K]$ is called a Condorcet winner of $\bP$ if $P(\a^*,\b) > 0.5$ for all $\b \in [K]\sm \{\a^*\}$ \citep{Zoghi+14RUCB,Komiyama+15,DTS, Busa21survey, VDB}.


Note that any preference matrix with a total ordering over arms invariably has a Condorcet winner. For example, assuming a total ordering $1 \succ 2 \succ \ldots \succ K$ implies that the Condorcet winner is arm~$1$. Any RUM-based preference matrix \cite{SG19,SG20,Az+13}, or more generally any $\bP$ that satisfies stochastic transitivity \citep{Yue+09}, always respects a total ordering. However, note that CW-based preference matrices consider a much bigger class of pairwise relations than total ordering. Despite this, in general a preference matrix might not have a Condorcet winner, which led to more general notions of benchmark arms in DB, such as the Borda winner \citep{ADB}, the Copeland winner \cite{Zoghi+15} or the von Neumann winner \citep{CDB,RDB}.  



\subsection{Non-Stationary Dueling Bandits (\nstdb)}
\label{subsec:nstdb}     

We consider a decision space of $K$ arms denoted by~$[K]$. At each round $t\in [T]$, the task of the learner is to select a pair of actions $(a_t, b_t) \in [K]\times [K]$, upon which a preference feedback $o_t(a_t, b_t) \sim \text{Ber}(P_t(a_t, b_t))$ is revealed to the learner according to the underlying preference matrix $\bP_t \in [0,1]^{K \times K}$. The sequence of preferences $\bP_1,\bP_2,\ldots, \bP_T$ is generated adversarially and for any such preference matrix $\bP_t$ we define
\begin{equation*}
    \delta_t (\a, \b) \coloneqq P_t (\a, \b) - 1/2
\end{equation*}
as the gap or preference-strength of arm $\a$ over arm $\b$ in round $t$. 
We here assume that every preference matrix $\bP_t$ has a Condorcet winner, which we denote by $\a_t^*$. 

\paragraph{Static Regret in Dueling Bandits.} In classical (stochastic) dueling bandits, where it is assumed that $\bP_1 = \ldots = \bP_T = \bP$ for some fixed preference matrix $\bP$, the performance of the learner is often measured w.r.t.\ the CW of $\bP$, described by the \emph{static regret} 
\begin{equation*}
\label{eq:sreg}
\SR(T) :=  \sum_{t=1}^T  \frac{\delta_t (\a^*, \a_t) + \delta_t(\a^*, \b_t) }{2}, 
\end{equation*}
where $\a^*$ is the CW of $\bP$ \citep{Yue+09, sui18survey,Busa21survey, SDB}. Note that here $\delta_t(\a^*, \a) = P_t(\a^*, \a) - 1/2$ essentially quantifies the net loss of arm $\a$ against the fixed benchmark arm~$\a^*$. 


However, regret with respect to any fixed benchmark (comparator arm) soon becomes meaningless when the underlying preference matrices are changing over time, since no single fixed arm may represent a reasonably good benchmark over $T$ rounds. Consider the following simple motivating example:
\begin{exmp}\label{eg:dyn_reg}

    Let $K=2$ and define 
    \begin{equation*}
    \bP_1 = 
    	\begin{bmatrix}
    		0.5 & 1\\
    		0 & 0.5
    	\end{bmatrix}
    , \hspace{0.5cm} \bP_2 = 
    	\begin{bmatrix}
    		0.5 & 0\\
    		1 & 0.5
    	\end{bmatrix} .
    \end{equation*}
    Now, assume a preference sequence such that $\bP_t = \bP_1$ for the first $\floor{T/2}$ rounds and $\bP_t = \bP_2$ for the last $\lceil T/2\rceil$ rounds. We see that a policy that plays any of the two arms all $T$ rounds, e.g.\ $\pi_t = 1$ for all $t \in [T]$, has regret $O(1)$ against any fixed benchmark arm, since $\delta_t(1, 2) = 1/2$ for the first $T/2$ rounds and $\delta_t (1, 2) = -1/2$ for last $T/2$ rounds.  
    However, against a \emph{dynamic benchmark}, e.g.\ arm $1$ for $t < T/2$ and arm $2$ for $t \geq T/2$, any policy that plays a fixed arm all $T$ rounds suffers $O(T/2)$ regret (while suffering only constant regret against any fixed benchmark).  
\end{exmp}

\textbf{Dynamic Regret in Dueling Bandits. } Drawing motivation from the above, we seek to formulate a stronger and more meaningful notion of dueling bandit regret, where the benchmark in every round is chosen dynamically based on $\bP_t$. More precisely, letting $\a_t^*$ be the CW of $\bP_t$, we define \emph{dynamic regret} as
\begin{equation*}
\label{eq:dreg}
	\DR(T) \coloneqq \sum_{t=1}^T  \frac{\delta_t (\a_t^*, \a_t) + \delta_t(\a_t^*, \b_t) }{2}. 
\end{equation*}

\subsection{Measures of Non-Stationarity}
\label{sec:nst_measures} 
Clearly, without any control over the amount of non-stationarity in the sequence $\{\bP_t\}_{t \in [T]}$, it is impossible for any learner to achieve sublinear $o(T)$ dynamic regret in the worst case. To see this, consider the matrices from Example~\ref{eg:dyn_reg} and note that for any choice of arms $(\a_t, \b_t)$, the adversary can choose a matrix so as to guarantee instantaneous regret of at least $1/2$. This consequently leads to linear regret for the learner, implying that to achieve sublinear dynamic regret, we need to restrict the adversary in terms of the total amount of non-stationarity it can induce in the sequence $\bP_1, \dots, \bP_T$. But what could be a good measure of non-stationarity? In this paper, we study several of these measures, which we will now formally introduce and put in relation to one another. 

\paragraph{1.\ \pv.} A non-stationarity measure that has been studied in the previous work on $\nstdb$ is the number of times $\bP_t$ changes~\citep{BengsNDB, SahaNDB}: 
\begin{align*}
    \Sp \coloneqq \sum_{t=2}^T \1\{\bP_t \not= \bP_{t-1} \}.
\end{align*}
However, $\Sp$ can be a quite pessimistic measure of non-stationarity, as changes in the preference between two suboptimal arms or minor preference shifts that do not change the CW are counted toward $\Sp$, whereas they should not significantly affect the performance of a good learning algorithm.

\paragraph{2.\ \cwv.} A naturally stronger measure of non-stationarity is the total number of Condorcet Winner Switches, i.e.\ the number of times the identity of $\a_t^*$ changes:
\begin{align*}
\Sw \coloneqq \sum_{t=2}^T \1\{\a_t^* \not= \a_{t-1}^*\}.   
\end{align*}

\begin{rem}[$\Sp$ vs $\Sw$]
\label{rem:cwv_vs_pv}
Of course, we always have $\Sw \leq \Sp$. In fact, it is easy to construct a simple scenario where $S^{\texttt{CW}} \ll S^{\texttt{P}}$: 
Assume $K=3$ and consider the following two preference matrices
\begin{align*}
\bP_1 = 
	\begin{bmatrix}
		0.5 & 0.55 & 0.55\\
		0.45 & 0.5 & 1\\
		0.45 & 0 & 0.5
	\end{bmatrix},  \hspace{0.5cm}
\bP_2 = 
	\begin{bmatrix}
		0.5 & 0.55 & 0.55\\
		0.45 & 0.5 & 0\\
		0.45 & 1 & 0.5
	\end{bmatrix} ,
\end{align*}
and a preference sequence such that $\bP_t = \bP_1$ when t is odd and $\bP_t = \bP_2$ otherwise. We then find that $S^{\texttt{CW}} = 0$ (since $1$ is the CW in all rounds $t$), whereas $S^{\texttt{P}} = T$. 
\end{rem}

\paragraph{3.\ \sigv.} Recently, \cite{Suk22} proposed a new (and strong) notion of non-stationarity for multi-armed bandits, called \emph{Significant Shifts}, that aims to account only for severe distribution shifts and comprises previous complexity measures. 
We can define a similar concept for dueling bandits:
Let $\htau_0 \coloneqq 1$ and define $\htau_{i+1}$ recursively as the first round in $[\htau_i, T)$ such that for all arms $a \in [K]$ there exist rounds $\htau_i \leq s_1 < s_2 < \htau_{i+1}$ such that $$\sum_{t=s_1}^{s_2} \delta_t(a_t^*, a) \geq  \sqrt{K(s_2-s_1)}.$$ 
Let $\sigS$ denote the number of such \emph{Significant CW Switches} $\htau_1, \dots, \htau_{\sigS}$. 
We immediately see that we have $\sigS \leq \Sw$, since not all CW Switches are also Significant CW Switches. For example, a 'non-severe' and quickly reverted change of the Condorcet winner may not be counted towards~$\sigS$. 

\paragraph{4.\ \mv.} Another common notion of non-stationarity studied in the multi-armed bandits literature is the total variation in the rewards~\cite{besbes+14,luo+18}. Its analogue in dueling bandits can be defined as
\begin{align*}
    \Vt \coloneqq \sum_{t=2}^{T}  \max_{a, b \in [K]} \abs{P_t(a, b) - P_{t-1}(a, b)},
\end{align*}
which has been previously studied in \cite{SahaNDB}. However, $\Vt$ can also be a pessimistic measure of complexity, as it can be of order $O(T)$ even though the Condorcet winner remains fixed throughout all rounds. 

\paragraph{5.\ \cwmv.} We can then formulate a more refined version of total variation by accounting only for the maximal drift in the winning probabilities of the current Condorcet winner:  
\begin{align*}
    \Vcw \coloneqq \sum_{t=2}^{T} \max_{a\in [K]} \abs{P_t(a_t^*, a) - P_{t-1}(a_t^*, a)} .
\end{align*}

\begin{rem}[$\Vt$ vs $\Vcw$]\label{rem:tv_vs_v}
It is clear from the definition that $\Vcw \leq \Vt$. Moreover, we again see that the Condorcet Winner Variation can be much smaller than the Total Variation in the preference sequence, i.e.\ $\Vcw \ll \Vt$. For example, in the problem instance of Remark~\ref{rem:cwv_vs_pv}, we find that $\Vcw = 0$, whereas $\Vt =T$. Thus, a regret bound in terms of the Condorcet Winner Variation $\Vcw$ can potentially be much stronger. 
\end{rem}


\begin{algorithm*}[t]
  \caption{\mainalgo: 
  Adaptive Non-stationAry  CONdorcet Dueling Algorithm}
    \label{alg:anaconda}

  \begin{algorithmic}[1]
      \STATE \textbf{input:} horizon $T$
      \STATE $t \leftarrow 1$
      \WHILE{$t \leq T$} 
      \STATE $t_\ell \leftarrow t$ \COMMENT{start of the $\ell$-th episode}
      \STATE $\mA \leftarrow [K]$
       \FOR[set replay schedule]{$m \in \{ 2, \dots, 2^{\lceil \log(T) \rceil}\}$ and $s \in \{t_\ell+1, \dots, T\}$}
      \STATE Sample $B_{s,m} \sim \bern\Big(\frac{1}{\sqrt{m (s-t_\ell)}}\Big) $ \label{line:sample_replay_schedule}
      \ENDFOR
      \STATE Run $\base(t_\ell, T+1-t_\ell)$ \COMMENT{root replay in $\ell$-th episode}
      \vspace{.001cm}
      \ENDWHILE
\end{algorithmic}
\end{algorithm*}

\begin{algorithm*}[t]
  \caption{$\base (\tstart, \mstart$)}
  \label{alg:replay}
  \begin{algorithmic}[1]
      \STATE \textbf{input:} scheduled time $\tstart$ and duration $\mstart$ 
      \STATE \textbf{initialize:} $t \leftarrow \tstart$, $\cA_t \leftarrow [K]$ 
      \WHILE[restart if no good arms are left]{$t \leq T$ and $t \leq t_0 + m_0$ and $\mA \neq \emptyset$}
          
      
      \STATE Play arm-pair $(a_t, b_t) \in \cA_t$ with each arm being selected with probability ${1}/{|\cA_t|}$ \label{line:sample_from_At} \label{line:sample_pair}
      
    \STATE $\mA \leftarrow \mA \setminus \{ a \in [K] \colon \exists [s_1, s_2] \subseteq [t_\ell, t) \text{ s.t.\ \eqref{eq:elim} holds}\}$ \COMMENT{eliminate bad arms from $\mA$} \label{line:elim_from_good}

      \STATE $\curA \leftarrow \cA_t$  \COMMENT{save active set of arms locally}
      \STATE $t \leftarrow t+1$

      \IF[check for scheduled child replays]{$\exists m$ such that $B_{t,m} = 1$}
      \STATE Run $\base(t, m)$ with $m = \max\{ m \in \{2, \dots, 2^{\lceil \log(T) \rceil}\} \colon B_{t,m} = 1\}$ 
      \ENDIF
      \STATE $\cA_t \leftarrow \curA \setminus \{ a \in [K] \colon \exists [s_1, s_2] \subseteq [\tstart, t) \text{ s.t.\ \eqref{eq:elim} holds}\}$ \COMMENT{eliminate bad arms from $\cA_t$} \label{line:elim_from_active}
    
      \ENDWHILE
\end{algorithmic}
\end{algorithm*}




\section{Proposed Algorithm: \mainalgo}\label{sec:algo}
Following recent advances in non-stationary multi-armed bandits \cite{auer+19, luo+19, Abasi2022A} and especially \citep{Suk22}, we construct an episode-based algorithm with a carefully chosen replay schedule, called $\mainalgo$. 

Recall that our goal is to minimize dynamic regret w.r.t.\ a changing benchmark $a_t^*$. However, we quickly notice that we cannot reliably track the dynamic regret of some arm $a \in [K]$, i.e.\ $\sum_t \delta_t(a_t^*, a)$, as the identity of the benchmark, $a_t^*$, changes at unknown times.  
As a resolution to this, we aim to detect relevant changes in the preference matrix by tracking the \emph{static regret} $\max_{a' \in [K]} \sum_{t=s_1}^{s_2} \delta_t (a', a)$ instead.
It will be the main challenge of our analysis to ensure that properly timed replays will occur (and not too many of these) so that it is in fact sufficient to track the static regret to guarantee low dynamic regret. 

In the following, we explain our algorithmic approach in more detail. 
The algorithm is organized in episodes, denoted $\ell$.
Similar to recent approaches to non-stationary multi-armed bandits \citep{auer+19, Abasi2022A, Suk22}, the algorithm maintains a set of good arms, $\mA$, and a replay schedule, $\{B_{s,m}\}_{s,m}$, within each episode. When no good arms are left in $\mA$, a new episode begins and the set of good arms and the replay schedule are being reset. 
Here, $\mainalgo$ (Algorithm~\ref{alg:anaconda}) is the meta procedure that initializes each episode by resetting the set of good arms to $[K]$, sampling a new replay schedule, and triggering the root call of $\base(t_\ell, T+1-t_\ell)$. 

When active in round $t$, a run of $\base(t_0, m_0)$ (Algorithm~\ref{alg:replay}) samples two arms uniformly at random from the active set of arms at round $t$, denoted $\cA_t$. The set $\cA_t$ is globally maintained by all calls of $\base$ and reset to $[K]$ at the beginning of each replay, i.e.\ call of $\base$. When a child replay $\base(t, m)$ is scheduled in round $t$, i.e.\ $B_{t, m} = 1$ for some $m$, the parent algorithm, say $\base(t_0, m_0)$, is interrupted (before eventually resuming if $t \leq t_0+m_0$ and $\mA \neq \emptyset$). To not overwrite arm eliminations of a parent by resetting $\cA_t$ to $[K]$ in interrupting calls of $\base$, each version of $\base$ saves a local set of arms, $\curA$, before checking for children. 




\paragraph{Gap Estimates.} Recall the definition of the gap between two arms as $\delta_t(a, b) = P_t(a, b) - 1/2$. Based on observed outcomes of duels, $\mainalgo$ maintains the following importance weighted estimates of $\delta_t(a, b)$: 
\begin{equation}\label{eq:estimates}
    \hdelta_t(a, b) = \abs{\cA_t}^2 \1_{\{a_t=a, b_t =b\}} o_t(a, b)  - 1/2. 
\end{equation}
Wee see that whenever $a, b \in \cA_t$, i.e.\ both arms are in the active set in round $t$, the estimator $\hdelta_t(a, b)$ is an unbiased estimate of $\delta_t(a, b)$, as we select a pair of arms uniformly at random from $\cA_t$ every round (see Line~\ref{line:sample_from_At} in Algorithm~\ref{alg:replay}).

\paragraph{Elimination Rule.} 
In Line~\ref{line:elim_from_good} and Line~\ref{line:elim_from_active} of Algorithm~\ref{alg:replay}, we eliminate an arm $a \in [K]$ in round $t$ if there exist rounds $0 \leq s_1 < s_2 \leq t$ such that 
\begin{equation}\label{eq:elim}
    \max_{a' \in [K]} \sum_{t=s_1}^{s_2} \hdelta_t(a', a) > C \log(T) K \sqrt{(s_2-s_1) \vee K^2} , 
\end{equation}
where $C > 0$ is some universal constant that does not depend on $T$, $K$, or $\Sw$, and can be derived from the regret analysis. 

\subsection{Main Result}

The main result of this paper is a $\tO(\sqrt{\Sw T})$ dynamic regret bound of $\mainalgo$ without knowledge of the number of CW Switches $\Sw$. When $\Sw \ll \Sp$, this bound substantially improves upon the \emph{non-adaptive} $\tO(\sqrt{\Sp T})$ rates in \cite{SahaNDB} and \cite{BengsNDB}. In particular, as previously mentioned, the number of preference switches $\Sp$ can be a very pessimistic measure of complexity. 
For example, a change in the preference between two suboptimal arms, or a minor change of the winning probabilities of the Condorcet winner under which it remains optimal, should not substantially affect our performance (see Remark~\ref{rem:cwv_vs_pv}).   


\begin{thm}[Dynamic Regret of \mainalgo]
\label{thm:main_result}
    Let $\Sw$ denote the unknown number of Condorcet Winner Switches. Let $\tau_1, \dots, \tau_{\Sw}$ be the unknown times of these switches and let $\tau_0 \coloneqq 1$ and $\tau_{\Sw+1} \coloneqq T$. For some constant $c >0$, the dynamic regret of $\mainalgo$ is bounded as   
    \begin{align*}
        \DR(T) \leq c\log^3(T) K \sum_{i=0}^{\Sw} \sqrt{\tau_{i+1}- \tau_i}. 
    \end{align*}
\end{thm}
An application of Jensen's inequality shows that this implies a dynamic regret bound of order $\tO(K \sqrt{\Sw T})$, stated in the following corollary. 
\begin{cor}[Dynamic Regret w.r.t.\ $\Sw$]
\label{cor:main}
    For some constant $c > 0$, the dynamic regret of $\mainalgo$ is bounded as 
    \begin{equation*}
        \DR(T) \leq c \log^3(T) K \sqrt{(\Sw +1) T}. 
    \end{equation*}
    
\end{cor} 

\vspace{0.1cm}

\begin{rem}[Regret Lower Bound and Tightness of \cref{thm:main_result}]
\label{rem:lb}
Note that a lower bound of $\Omega(\sqrt{K \Sp T})$ has recently been shown by \cite{SahaNDB}, which can also be seen to give a lower bound $\Omega(\sqrt{K \Sw T})$ in terms of CW Switches $\Sw$ as $\Sw \leq \Sp$ (in particular, the lower bound problem instance used in \cite{SahaNDB} is precisely such that $\Sw = \Sp$). As a result, we find that the above bound is optimal up to logarithmic factors in its dependence on $\Sw$ and $T$, whereas its dependence on $K$ may not be tight.  
\end{rem}

\section{Regret Analysis of \mainalgo}\label{sec:proof_sketch}

We build on recent advances in non-stationary multi-armed bandits, which are able to achieve near-optimal dynamic guarantees~\citep{auer+19, Abasi2022A, Suk22} without knowledge of the non-stationary complexity.  
A common basis of the regret analysis in these works is a decomposition of the dynamic regret using the notion of good arms. 

\paragraph{Challenges in the Dueling Setting.} 
More precisely, within each episode $\ell$, prior work in multi-armed bandits \cite{auer+19, Abasi2022A, Suk22} decomposes the regret of their algorithm's selection, say, $a_t$ into its relative regret against the last good arm $\agood \in \mA$, and the relative regret of $\agood$ against the best arm, say, $a_t^*$. A key advantage of this decomposition is that estimating the relative regret of some arm $a$ w.r.t.\ $\agood$ instead of $a_t^*$ is much easier. 
In particular, since $\agood$ is by definition considered good throughout the episode, it is always actively played, which guarantees unbiased estimates of the difference in rewards between any played arm $a$ and the last good arm $\agood$.

However, pairwise preferences are generally not transitive, let alone linear, so that a triangle inequality does not hold, i.e.\ $\delta_t(a_t^*, a) \not \leq \delta_t(a_t^*, \agood) + \delta_t(\agood, a)$. 
In $\nstdb$, we can thus generally not utilize $\agood$, or any other temporarily fixed arm, as a benchmark to detect large regret. Instead, in contrast to prior work in multi-armed bandits, we face the difficulty of having to argue directly that we can guarantee low dynamic regret $\sum_t \delta_t(a_t^*, a)$ without a proxy benchmark such as~$\agood$. 


\paragraph{Key Ideas to Overcome these Challenges.} 
To overcome these challenges, we consider every fixed arm $a \in [K]$ in isolation and split each episode $\ell$ into the rounds before arm $a$ gets eliminated from $\mA$ and the rounds after it gets eliminated from $\mA$. Letting $t_\ell^a$ be the elimination round of arm $a$, we will then argue that $t_\ell^a$ will occur sufficiently early to guarantee low regret (in episode $\ell$) before round $t_\ell^a$. For the rounds after elimination from $\mA$, it will be key to dissect each possible replay of the eliminated arm and obtain replay-specific regret bounds, where we distinguish between  'confined' and 'unconfined' replays of arms.  We now give an outline of our regret analysis. 

\subsection{Proof Sketch of \cref{thm:main_result}}
In the following, we let $\const > 0$ denote a positive constant that does not dependent on $T$, $K$, or $\Sw$, but may change from line to line. 
To begin our analysis, we state a concentration bound on the martingale difference sequence $\hdelta_t(a, b) - \E [ \hdelta_t(a, b) \mid \cF_{t-1}] $ as it can be found in similar form in \citep{beygelzimer2011contextual} and~\citep{Suk22}. 
\begin{lem}\label{lem:concentration_proof_sketch}
    Let $\cE$ be the event that for all rounds $1 \leq s_1 < s_2 \leq T$ and all arms $a, b \in [K]$:
    \begin{align}
        \label{eq:concentr_proof_sketch}
        \abs{\sum_{t=s_1}^{s_2} \hdelta_t(a, b) - \sum_{t=s_1}^{s_2} \E \left[ \hdelta_t(a, b) \mid \cF_{t-1}\right]} \leq \const \log(T) \left( K\sqrt{(s_2-s_1)} + K^2\right)
    \end{align}
    for a sufficiently large constant $\const >0$ and where $\cF = \{ \cF_t\}_{t \in \N_0}$ denotes the canonical filtration. Then, event $\cE$ occurs with probability at least $1-1/T^2$. 
\end{lem} 
Note that our elimination rule~\eqref{eq:elim} has been chosen in accordance with the above concentration bound.
In particular, let $t_\ell^a$ denote the round in episode $\ell$ in which arm $a$ is eliminated from $\mA$. Then, on the concentration event $\cE$, if $a' \in \mA$ for all $t_\ell \leq t < t_\ell^a$, we must have 
\begin{align*}
    \sum_{t=t_\ell}^{t_\ell^a -1} \delta_t(a', a)
    & = \sum_{t=t_\ell}^{t_\ell^a-1} \E \big[\hdelta_t(a', a) \mid \cF_{t-1} \big] \leq \const \log(T) K \sqrt{(t_\ell^a -t_\ell) \vee K^2}, 
\end{align*}
where the initial identity holds as $\hdelta_t(a', a)$ is unbiased when $a, a' \in \cA_t$ and the inequality follows from the elimination rule~\eqref{eq:elim} and the concentration bound~\eqref{eq:concentr_proof_sketch}. 
However, note that the above crucially used that both $a$ and $a'$ are actively played throughout the interval $[t_\ell, t_\ell^a)$, as we are otherwise not able to accurately estimate $\sum_{t} \delta_t(a', a)$. It will be the primary challenge of our analysis to ensure that through properly timed replays, i.e.\ calls of $\base$, we can obtain unbiased estimates w.r.t.\ the changing CW that allow us to eliminate bad arms before they amass large regret. 

\paragraph{Bounding Regret Within Episodes.} 
We proceed by bounding regret within each episode separately. 
Recall that we let $\tau_1 < \ldots < \tau_{\Sw}$ denote the (unknown) rounds in which the Condorcet winner changes. We then refer to the interval $[\tau_i, \tau_{i+1})$ as the $i$-th phase, i.e.\ the interval for which $a_t^* = a_{\tau_i}^*$ for all $t\in [\tau_i, \tau_{i+1})$. 
Let $\phases(t_1, t_2) = \{ i : [\tau_i, \tau_{i+1}) \cap [t_1, t_2) \neq \emptyset\}$ be the set of phases $i$ such that $[\tau_i, \tau_{i+1})$ intersects with the interval $[t_1, t_2)$. 
Our main claim is the following upper bound on the dynamic regret within each episode:
\begin{align}
    \label{eq:sketch_proof_main_claim}
    & \E\left[ \sum_{t=t_\ell}^{t_{\ell+1}-1} \frac{\delta_t(a_t^*, a_t)  + \delta_t(a_t^*, b_t)}{2} \right] \leq \const K \log^3(T) \, \E\left[\sum_{i \in \phases(t_\ell, t_{\ell+1})} \sqrt{\tau_{i+1}- \tau_i} \right].  
\end{align}

By conditioning on $t_\ell$ and carefully applying the tower property, we can rewrite the expected dynamic regret within an episode in terms of fixed arms $a \in [K]$:
\begin{lem}
    We have 
    \begin{align*}
        & \E\left[ \sum_{t=t_\ell}^{t_{\ell+1}-1} \frac{\delta_t(a_t^*, a_t)  + \delta_t(a_t^*, b_t)}{2} \right] = \E\left[ \sum_{a=1}^K \sum_{t=t_\ell}^{t_{\ell+1}-1} \frac{\delta_t(a_t^*, a)}{\abs{\cA_t}} \1_{\{a \in \cA_t\}} \right]. 
    \end{align*}
    
\end{lem} 

In a next step, we split the RHS into the rounds before a fixed arm $a \in [K]$ has been eliminated from the good set, and the rounds after its elimination. Recall $t_\ell^a$ to be the round in episode $\ell$ in which arm $a$ is eliminated from $\mA$ and consider 
\begin{align*}
    \underbrace{\E\Bigg[ \sum_{a=1}^K \sum_{t=t_\ell}^{t_{\ell}^a-1} \frac{\delta_t(a_t^*, a)}{\abs{\cA_t}} \Bigg]}_{\terma}+ \underbrace{\E \Bigg[ \sum_{a=1}^K \sum_{t=t_\ell^a}^{t_{\ell+1}-1} \frac{\delta_t(a_t^*, a)}{\abs{\cA_t}}  \1_{\{a \in \cA_t\}} \Bigg]}_{\termb} ,
\end{align*}
where we could drop the indicator in $\terma$, since $\mA \subseteq \cA_t$ by construction of these sets.  
The remainder of our analysis is mostly concerned with showing that both, $\terma$ and $\termb$, are upper bounded by the RHS in~\eqref{eq:sketch_proof_main_claim}.

\paragraph{Regret Before Elimination.} 
The main difficulty in bounding $\terma$ lies in the fact that some arm could have been eliminated due to being suboptimal, only to become the Condorcet winner shortly after. As a result, large regret could go undetected, as the current Condorcet winner is not being actively played anymore. 
To this end, we have to argue that with high probability there will always be a replay scheduled that eliminates any bad arm from $\mA$ in a timely manner, thereby eventually triggering a restart.

Here, we specifically consider calls of $\base(s, m)$ that provably eliminate bad arms from $\mA$. Importantly, by construction of our elimination rule~\eqref{eq:elim}, we can guarantee on the concentration event $\cE$ that any run of $\base$ that is scheduled within some phase $i$ will actively play the Condorcet winner of said phase. 
\begin{lem}\label{lem:proof_sketch_no_elim}
    On event $\cE$, no call of $\base(s, m)$ with $\tau_i \leq s < \tau_{i+1}$ eliminates arm $a_i^*$ before round~$\tau_{i+1}$. 
\end{lem} 
Roughly speaking, we can then argue that a replay that eliminates arm $a$ will be scheduled with high probability before the smallest round $s(a) > t_\ell$ such that $\sum_{t=t_\ell}^{s(a)} \delta_t (a_t^*, a) \gtrsim \sqrt{s(a) - t_\ell}$. In other words, arm $a$ is going to be eliminated from $\mA$ before it suffers too much regret. 
Since $t_\ell^a$ is defined as the round in episode $\ell$ in which $a$ is eliminated from $\mA$, we must have $t_\ell^a < s(a)$, which implies that the inner sum in $\terma$ is at most of order $\sqrt{t_\ell^a - t_\ell}$ for every fixed arm $a \in [K]$. Finally, using that 
\begin{equation*}
    \sqrt{t_\ell^a - t_\ell} \leq \sum_{i\in \phases(t_\ell, t_\ell^a)} \sqrt{\tau_{i+1}-\tau_i} 
\end{equation*}
and summing over all arms, we obtain the desired bound of~\eqref{eq:sketch_proof_main_claim}. Note that here summing over arms can be seen to account for a $\log(K)$ factor which we coarsely upper bound by $\log(T)$. 
 
\paragraph{Regret After Elimination.} 
$\termb$ can be viewed as the regret due to replaying arms after they have been eliminated from the good set $\mA$. We here distinguish between two types of replays, i.e.\ calls of $\base$:
\begin{defn}
    We call $\base(s, m)$ \emph{confined} if there exists $i \in \phases(t_\ell, T)$ s.t.\ $[s, s+m] \subseteq [\tau_i, \tau_{i+1})$. In turn, we say that $\base(s, m)$ is \emph{unconfined} if for all $i \in \phases(t_\ell, T)$, we have $[s, s+m] \subseteq [\tau_i, \tau_{i+1})$. 
\end{defn} 
To bound the regret within a confined replay, we recall that according to Lemma~\ref{lem:proof_sketch_no_elim}, on the concentration event $\cE$, no call of $\base$ will eliminate the Condorcet winner within the phase it is scheduled in. Thus, whenever some arm $a$ is being played by a confined replay, we obtain unbiased estimates of $\delta_t(a_t^*, a)$. It is then straightforward to show that for any confined $\base(s, m)$, we have that $\sum_{t=s}^{s+m} \delta_t(a_t^*, a)$ is at most of order $\sqrt{m}$. 

A similar line of argument does not work for unconfined replays, as they intersect with several phases. We then face a similar difficulty as when bounding $\terma$, where the Condorcet winner of the current phase could have been eliminated (from the replay) in an earlier phase. Using similar arguments than for bounding $\terma$, we show that for any unconfined $\base(s, m)$, we have that $\sum_{t=s}^{s+m} \delta_t(a_t^*, a)$ is at most of order $\sqrt{s-t_\ell} + \sqrt{m}$. 

Lastly, recall that in episode $\ell$ a replay $\base(s, m)$ is scheduled with probability ${1}/{\sqrt{m(s-t_\ell)}}$. Crucially, any unconfined $\base$ scheduled in $[\tau_i, \tau_{i+1})$ must have duration at least $m \geq \tau_{i+1} - s$ (otherwise it is not unconfined). Careful summation over confined and unconfined $\base$ then yields the desired upper bound~\eqref{eq:sketch_proof_main_claim}.

\paragraph{Counting Episodes.} 
Lastly, we show that $\mainalgo$ only restarts if there has been a CW switch. 
\vspace{-0.2cm}
\begin{lem}\label{lem:proof_sketch_phase_intersect}
    On event $\cE$, for all episodes $\ell$ but the last there exists a change of the CW $t_\ell \leq \tau_i < t_{\ell+1}$. 
\end{lem}
This follows directly from the fact that on the concentration event within a single phase the CW will never be eliminated from $\mA$. Thus, if there is a restart, i.e.\ every arm has been eliminated from $\mA$, there must have been a change of CW. 
Lemma~\ref{lem:proof_sketch_phase_intersect} thus tells us that any phase intersects with at most two episodes. Summing the RHS of~\eqref{eq:sketch_proof_main_claim} over episodes then gives the claimed upper bound of 
\begin{align*}
    & \E \left[ \sum_{t=1}^T \frac{\delta_t(a_t^*, a_t) + \delta_t(a_t^*, b_t)}{2} \right] \leq 2 \const K \log^3(T) \E \left[ \sum_{i=1}^{\Sw} \sqrt{\tau_{i+1} - \tau_i} \right]. 
\end{align*}
A detailed proof of Theorem~\ref{thm:main_result} is given in Appendix~\ref{sec:appendix_proof_thm}.

\section{Tighter Bounds Under SST and STI} 
\label{sec:sign_switches}

We show that \mainalgo\, can in fact yield a stronger regret guarantee in terms of a more refined notion of non-stationarity, \sigv\, (see \cref{sec:nst_measures}), under additional assumptions on the preference sequence $\bP_1, \dots, \bP_T$: {Strong Stochastic Transitivity} (SST) and {Stochastic Triangle Inequality} (STI) \citep{Yue+09,Yue+12,BTM}. 
Let $a, b, c \in [K]$ and let $a \succ_t b$ denote that $a$ is preferred over $b$ in round~$t$. 

\begin{assumption}[Strong Stochastic Transitivity] 
    \label{asm:sst}
    Every preference matrix $\bP_t$ satisfies that if $a \succ_t b \succ_t c$, we have $\delta_t(a,c) \geq \delta_t(a, b) \vee  \delta_t(b, c)$.  
\end{assumption}

\begin{assumption}[Stochastic Triangle Inequality]
    \label{asm:sti}
    Every preference matrix $\bP_t$ satisfies that if $a \succ_t b \succ_t c$, we have $\delta_t(a,c) \leq \delta_t(a, b) +   \delta_t(b, c)$.  
\end{assumption} 


\begin{rem}[Example of SST \& STI]
\label{rem:eg_sti}
    Among the preference models that satisfy \cref{asm:sst} and \cref{asm:sti}, are utility-based models with a symmetric and monotonically increasing link function $\sigma$. In these models, every arm $a$ has an associated (time-dependent) utility $u_t(a)$ and the probability of arm $a$ winning a duel against arm $b$ is given by $P_t(a \succ b) = \sigma(u_t(a) - u_t(b))$, where $\sigma$ is an increasing function satisfying $\sigma(x) = 1-\sigma(-x)$ and $\sigma(0) = 1/2$ that maps utility differences to probabilities~\cite{Yue+12, Busa21survey}.     
\end{rem} 

\subsection{Improved Dynamic Regret Analysis}
We now show that $\mainalgo$ achieves strong regret guarantees in terms of Significant CW Switches and CW Variation under SST and STI. 
 
\paragraph{Significant Condorcet Winner Switches.} Under \cref{asm:sst} and \cref{asm:sti}, we are able to obtain the following adaptive dynamic regret bound in terms of $\sigS$.

\begin{thm}
\label{thm:sign_cw_changes}
    Let $\sigS$ be the unknown number of $\sigv$. Under Assumption~1 and Assumption~2, $\mainalgo$ has dynamic regret $\tO\big(K \sqrt{{\sigS} T}\big)$. 
\end{thm} 

\begin{rem}
\label{rem:rush}
Recall from \cref{sec:nst_measures}, since $\sigS \leq \Sw$ (as not all CW Switches are also Significant CW Switches), \cref{thm:sign_cw_changes} gives a tighter dynamic regret guarantee for the class of non-stationary preference sequences with SST and STI. Also note that this bound does not violate the $\Omega(\sqrt{K \Sp T})$  lower bound from \ref{rem:lb}, as the lower bound is shown for a worst-case preference sequence $\bP_1, \dots, \bP_T$ where $\sigS = \Sw = \Sp$.
\end{rem}

\begin{proof}[Proof Overview]
    With some additional effort, \cref{asm:sst} and \cref{asm:sti} allow us to utilize a dynamic regret decomposition similar to prior work in non-stationary multi-armed bandits~\cite{auer+19, Suk22, Abasi2022A}. Roughly speaking, this allows us to reuse the regret analysis for CW Switches (Theorem~\ref{thm:main_result}) in the analysis under Significant CW Switches.   
\end{proof}

We want to give a brief intuition about why additional assumptions are necessary when bounding dynamic regret w.r.t.\ Significant CW Switches $\sigS$ opposed to CW Switches $\Sw$.\footnote{Note that this is a limitation of our regret analysis. It is an open question whether it is possible to achieve ${O(\sqrt{\sigS T})}$ dynamic regret in $\nstdb$ with general preference models.} Consider a phase $[\htau_i, \htau_{i+1})$ in the sense of Significant CW Switches as defined in Section~\ref{sec:nst_measures}. As previously mentioned, the definition of a Significant CW Switch allows for several (non-severe) CW changes within each phase $[\htau_i, \htau_{i+1})$. As a result, we cannot guarantee that there will be any intervals during which the CW remains fixed, which would enable us to accurately estimate the relative regret $\sum_t \delta_t(a_t^*, a)$ so as to eliminate bad arms. Broadly speaking, assuming a sort of transitivity (i.e.\ SST and STI) enables us to identify bad arms based on knowledge of $\sum_t \delta_t(a', a)$ for some temporarily fixed benchmark $a'$. 
More details and a complete proof can be found in Appendix~\ref{sec:appendix_proof_sign_thm}.

\paragraph{\cv.} Recall the definition of the \cv\, $\Vcw$ from \cref{sec:nst_measures}. As a consequence of Theorem~\ref{thm:sign_cw_changes}, we can show that $\mainalgo$ also achieves near-optimal dynamic regret w.r.t.\ $\Vcw$.      
\begin{cor}
\label{cor:total_variation}
    Let $\Vcw$ be the unknown \cv. Under \cref{asm:sst} and \cref{asm:sti}, $\mainalgo$ has dynamic regret  $\tO\big( K\sqrt{T} + {\Vcw}^{\nicefrac{1}{3}} (KT)^{\nicefrac{2}{3}}\big)$. 
\end{cor} 

\begin{rem}
\label{rem:tv_under_STT_STI}
By definition, we have $\Vcw \leq \Vt$, which means that Corollary~\ref{cor:total_variation} may yield a tighter dynamic regret bound than the (non-adaptive) $\tO\big((K \Vt)^{\nicefrac{1}{3}} T^{\nicefrac{2}{3}}\big)$ guarantee in \cite{SahaNDB}. In view of the lower bound of $\Omega\big((K \Vt)^{\nicefrac{1}{3}} T^{\nicefrac{2}{3}}\big)$ shown in~\citep{SahaNDB}, the regret guarantee of $\mainalgo$ is also tight up to logarithmic factors and a factor of $K^{\nicefrac{1}{3}}$. Note again that the lower bound in~\citep{SahaNDB} is not violated as their lower bound uses a worst-case preference sequence $\bP_1, \dots, \bP_T$ where $\Vcw = \Vt$.  
\end{rem}

\section{Discussion}
\label{sec:concl}

We studied the problem of dynamic regret minimization in non-stationary dueling bandits and proposed an adaptive algorithm that yields provably optimal regret guarantees in terms of strong notions of non-stationary complexity. Our proposed algorithm is the first to achieve optimal dynamic dueling bandit regret without prior knowledge of the underlying non-stationary complexity. While our results certainly close some of the practical open problems in preference elicitation in time-varying preference models, it also leads to plethora of new questions along the line. We provide an outlook to future directions and open problems in the supplementary material.

\paragraph{Future Work.} While our results certainly address some of the practical open problems for preference elicitation in time-varying preference models, it also leads to plethora of new questions along the line. In particular, as an extension to this work, one obvious question would be to understand non-stationary dueling bandits for more general preference matrices: What happens if the preference sequences do not have a Condorcet winner in each round? What could be a good dynamic benchmark in that case? 
Hereto related, another open question is whether it is possible to obtain dynamic regret bounds in terms of Significant CW Switches ($\sigS$) for general preference sequences (without transitivity assumptions). 
Extending the considered pairwise preference setting to more general subsetwise feedback \cite{SG18,SGwin18,GS21,SGinst20} would be another interesting direction from a practical point of view.
%

\section*{Acknowledgement} 
Thomas Kleine Buening was supported by the Research Council of Norway, Grant No.\ 302203.

\bibliographystyle{plainnat}
\bibliography{nstDB,db_refs}

\begin{thebibliography}{50}
\providecommand{\natexlab}[1]{#1}
\providecommand{\url}[1]{\texttt{#1}}
\expandafter\ifx\csname urlstyle\endcsname\relax
  \providecommand{\doi}[1]{doi: #1}\else
  \providecommand{\doi}{doi: \begingroup \urlstyle{rm}\Url}\fi

\bibitem[Abbasi-Yadkori et~al.(2022)Abbasi-Yadkori, Gyorgy, and
  Lazic]{Abasi2022A}
Yasin Abbasi-Yadkori, Andras Gyorgy, and Nevena Lazic.
\newblock A new look at dynamic regret for non-stationary stochastic bandits.
\newblock \emph{arXiv preprint arXiv:2201.06532}, 2022.

\bibitem[Agrawal and Goyal(2012)]{TS12}
Shipra Agrawal and Navin Goyal.
\newblock Analysis of {T}hompson sampling for the multi-armed bandit problem.
\newblock In \emph{Conference on Learning Theory}, pages 39--1, 2012.

\bibitem[Ailon et~al.(2014)Ailon, Karnin, and Joachims]{Ailon+14}
Nir Ailon, Zohar~Shay Karnin, and Thorsten Joachims.
\newblock Reducing dueling bandits to cardinal bandits.
\newblock In \emph{ICML}, volume~32, pages 856--864, 2014.

\bibitem[Allesiardo et~al.(2017)Allesiardo, F{\'e}raud, and
  Maillard]{allesiardo2017non}
Robin Allesiardo, Rapha{\"e}l F{\'e}raud, and Odalric-Ambrym Maillard.
\newblock The non-stationary stochastic multi-armed bandit problem.
\newblock \emph{International Journal of Data Science and Analytics},
  3\penalty0 (4):\penalty0 267--283, 2017.

\bibitem[Audibert and Bubeck(2010)]{Audibert+10}
Jean-Yves Audibert and S{\'e}bastien Bubeck.
\newblock Best arm identification in multi-armed bandits.
\newblock In \emph{COLT-23th Conference on Learning Theory-2010}, pages 13--p,
  2010.

\bibitem[Auer et~al.(2002)Auer, Cesa-Bianchi, and Fischer]{Auer+02}
Peter Auer, Nicolo Cesa-Bianchi, and Paul Fischer.
\newblock Finite-time analysis of the multiarmed bandit problem.
\newblock \emph{Machine learning}, 47\penalty0 (2-3):\penalty0 235--256, 2002.

\bibitem[Auer et~al.(2019)Auer, Gajane, and Ortner]{auer+19}
Peter Auer, Prateek Gajane, and Ronald Ortner.
\newblock Adaptively tracking the best bandit arm with an unknown number of
  distribution changes.
\newblock \emph{In Proceedings of the 32nd International Conference on Learning
  Theory}, 99:\penalty0 138--158, 2019.

\bibitem[Bengs et~al.(2021)Bengs, Busa-Fekete, El~Mesaoudi-Paul, and
  H{\"u}llermeier]{Busa21survey}
Viktor Bengs, R{\'o}bert Busa-Fekete, Adil El~Mesaoudi-Paul, and Eyke
  H{\"u}llermeier.
\newblock Preference-based online learning with dueling bandits: A survey.
\newblock \emph{Journal of Machine Learning Research}, 2021.

\bibitem[Besbes et~al.(2014)Besbes, Gur, and Zeevi]{besbes+14}
Omar Besbes, Yonatan Gur, and Assaf Zeevi.
\newblock Stochastic multi-armed-bandit problem with non-stationary rewards.
\newblock \emph{Advances in Neural Information Processing Systems},
  27:\penalty0 199--207, 2014.

\bibitem[Besbes et~al.(2015)Besbes, Gur, and Zeevi]{besbes+15}
Omar Besbes, Yonatan Gur, and Assaf Zeevi.
\newblock Non-stationary stochastic optimization.
\newblock \emph{Operations research}, 63\penalty0 (5):\penalty0 1227--1244,
  2015.

\bibitem[Beygelzimer et~al.(2011)Beygelzimer, Langford, Li, Reyzin, and
  Schapire]{beygelzimer2011contextual}
Alina Beygelzimer, John Langford, Lihong Li, Lev Reyzin, and Robert Schapire.
\newblock Contextual bandit algorithms with supervised learning guarantees.
\newblock In \emph{Proceedings of the Fourteenth International Conference on
  Artificial Intelligence and Statistics}, pages 19--26. JMLR Workshop and
  Conference Proceedings, 2011.

\bibitem[Bubeck et~al.(2012)Bubeck, Cesa-Bianchi, et~al.]{BubeckNotes+12}
S{\'e}bastien Bubeck, Nicolo Cesa-Bianchi, et~al.
\newblock Regret analysis of stochastic and nonstochastic multi-armed bandit
  problems.
\newblock \emph{Foundations and Trends{\textregistered} in Machine Learning},
  5\penalty0 (1):\penalty0 1--122, 2012.

\bibitem[Chen et~al.(2019)Chen, Lee, Luo, and Wei]{luo+19}
Yifang Chen, Chung-Wei Lee, Haipeng Luo, and Chen-Yu Wei.
\newblock A new algorithm for non-stationary contextual bandits: Efficient,
  optimal, and parameter-free.
\newblock \emph{In Proceedings of the 32nd Conference on Learning Theory},
  99:\penalty0 1--30, 2019.

\bibitem[Dud{\'\i}k et~al.(2015)Dud{\'\i}k, Hofmann, Schapire, Slivkins, and
  Zoghi]{CDB}
Miroslav Dud{\'\i}k, Katja Hofmann, Robert~E Schapire, Aleksandrs Slivkins, and
  Masrour Zoghi.
\newblock Contextual dueling bandits.
\newblock In \emph{Conference on Learning Theory}, pages 563--587, 2015.

\bibitem[Dudik et~al.(2015)Dudik, Hofmann, Schapire, Slivkins, and
  Zoghi]{dudik+15}
Miroslav Dudik, Katja Hofmann, Robert~E Schapire, Aleksandrs Slivkins, and
  Masrour Zoghi.
\newblock Contextual dueling bandits.
\newblock \emph{Conference on Learning Theory}, pages 563--587, 2015.

\bibitem[Gajane et~al.(2015)Gajane, Urvoy, and Cl{\'e}rot]{Adv_DB}
Pratik Gajane, Tanguy Urvoy, and Fabrice Cl{\'e}rot.
\newblock A relative exponential weighing algorithm for adversarial
  utility-based dueling bandits.
\newblock In \emph{Proceedings of the 32nd International Conference on Machine
  Learning}, pages 218--227, 2015.

\bibitem[Garivier and Moulines(2011)]{garivier2011upper}
Aur{\'e}lien Garivier and Eric Moulines.
\newblock On upper-confidence bound policies for switching bandit problems.
\newblock In \emph{International Conference on Algorithmic Learning Theory},
  pages 174--188. Springer, 2011.

\bibitem[Ghoshal and Saha(2022)]{GS21}
Suprovat Ghoshal and Aadirupa Saha.
\newblock Exploiting correlation to achieve faster learning rates in low-rank
  preference bandits.
\newblock In \emph{International Conference on Artificial Intelligence and
  Statistics}, pages 456--482. PMLR, 2022.

\bibitem[Gupta and Saha(2022)]{SahaNDB}
Shubham Gupta and Aadirupa Saha.
\newblock Optimal and efficient dynamic regret algorithms for non-stationary
  dueling bandits.
\newblock In \emph{International Conference on Machine Learning}, pages
  19027--19049. PMLR, 2022.

\bibitem[Hofmann et~al.(2011)Hofmann, Whiteson, and
  De~Rijke]{hofmann2011probabilistic}
Katja Hofmann, Shimon Whiteson, and Maarten De~Rijke.
\newblock A probabilistic method for inferring preferences from clicks.
\newblock In \emph{Proceedings of the 20th ACM international conference on
  Information and knowledge management}, pages 249--258, 2011.

\bibitem[Kolpaczki et~al.(2022)Kolpaczki, Bengs, and H{\"u}llermeier]{BengsNDB}
Patrick Kolpaczki, Viktor Bengs, and Eyke H{\"u}llermeier.
\newblock Non-stationary dueling bandits.
\newblock \emph{arXiv preprint arXiv:2202.00935}, 2022.

\bibitem[Komiyama et~al.(2015)Komiyama, Honda, Kashima, and
  Nakagawa]{Komiyama+15}
Junpei Komiyama, Junya Honda, Hisashi Kashima, and Hiroshi Nakagawa.
\newblock Regret lower bound and optimal algorithm in dueling bandit problem.
\newblock In \emph{COLT}, pages 1141--1154, 2015.

\bibitem[Lattimore and Szepesv{\'a}ri(2018)]{CsabaNotes18}
Tor Lattimore and Csaba Szepesv{\'a}ri.
\newblock Bandit algorithms.
\newblock \emph{preprint}, 2018.

\bibitem[Luo et~al.(2018)Luo, Wei, Agarwal, and Langford]{luo+18}
Haipeng Luo, Chen-Yu Wei, Alekh Agarwal, and John Langford.
\newblock Efficient contextual bandits in non-stationary worlds.
\newblock \emph{In Proceedings of the 31st Conference On Learning Theory},
  75:\penalty0 1739--1776, 2018.

\bibitem[Radlinski and Craswell(2013)]{radlinski2013optimized}
Filip Radlinski and Nick Craswell.
\newblock Optimized interleaving for online retrieval evaluation.
\newblock In \emph{Proceedings of the sixth ACM international conference on Web
  search and data mining}, pages 245--254, 2013.

\bibitem[Robbins(1952)]{robbins1952some}
Herbert Robbins.
\newblock Some aspects of the sequential design of experiments.
\newblock \emph{Bulletin of the American Mathematical Society}, 58\penalty0
  (5):\penalty0 527--535, 1952.

\bibitem[Russac et~al.(2019)Russac, Vernade, and Capp{\'e}]{cappe1}
Yoan Russac, Claire Vernade, and Olivier Capp{\'e}.
\newblock Weighted linear bandits for non-stationary environments.
\newblock \emph{Advances in Neural Information Processing Systems}, 32, 2019.

\bibitem[Russac et~al.(2020)Russac, Capp{\'e}, and Garivier]{cappe2}
Yoan Russac, Olivier Capp{\'e}, and Aur{\'e}lien Garivier.
\newblock Algorithms for non-stationary generalized linear bandits.
\newblock \emph{arXiv preprint arXiv:2003.10113}, 2020.

\bibitem[Saha and Gaillard(2021)]{SDB}
Aadirupa Saha and Pierre Gaillard.
\newblock Dueling bandits with adversarial sleeping.
\newblock \emph{Advances in Neural Information Processing Systems},
  34:\penalty0 27761--27771, 2021.

\bibitem[Saha and Gaillard(2022)]{VDB}
Aadirupa Saha and Pierre Gaillard.
\newblock Versatile dueling bandits: Best-of-both-world analyses for online
  learning from preferences.
\newblock In \emph{International Conference on Machine Learning}. PMLR, 2022.

\bibitem[Saha and Gopalan(2018)]{SG18}
Aadirupa Saha and Aditya Gopalan.
\newblock Battle of bandits.
\newblock In \emph{Uncertainty in Artificial Intelligence}, 2018.

\bibitem[Saha and Gopalan(2019{\natexlab{a}})]{SG19}
Aadirupa Saha and Aditya Gopalan.
\newblock Combinatorial bandits with relative feedback.
\newblock In \emph{Advances in Neural Information Processing Systems},
  2019{\natexlab{a}}.

\bibitem[Saha and Gopalan(2019{\natexlab{b}})]{SGwin18}
Aadirupa Saha and Aditya Gopalan.
\newblock {PAC Battling Bandits in the Plackett-Luce Model}.
\newblock In \emph{Algorithmic Learning Theory}, pages 700--737,
  2019{\natexlab{b}}.

\bibitem[Saha and Gopalan(2020{\natexlab{a}})]{SG20}
Aadirupa Saha and Aditya Gopalan.
\newblock Best-item learning in random utility models with subset choices.
\newblock In \emph{International Conference on Artificial Intelligence and
  Statistics}, pages 4281--4291. PMLR, 2020{\natexlab{a}}.

\bibitem[Saha and Gopalan(2020{\natexlab{b}})]{SGinst20}
Aadirupa Saha and Aditya Gopalan.
\newblock From pac to instance-optimal sample complexity in the plackett-luce
  model.
\newblock In \emph{International Conference on Machine Learning}, pages
  8367--8376. PMLR, 2020{\natexlab{b}}.

\bibitem[Saha and Krishnamurthy(2022)]{RDB}
Aadirupa Saha and Akshay Krishnamurthy.
\newblock Efficient and optimal algorithms for contextual dueling bandits under
  realizability.
\newblock In \emph{International Conference on Algorithmic Learning Theory},
  pages 968--994. PMLR, 2022.

\bibitem[Saha et~al.(2021)Saha, Koren, and Mansour]{ADB}
Aadirupa Saha, Tomer Koren, and Yishay Mansour.
\newblock Adversarial dueling bandits.
\newblock In \emph{International Conference on Machine Learning}, pages
  9235--9244. PMLR, 2021.

\bibitem[Soufiani et~al.(2013)Soufiani, Parkes, and Xia]{Az+13}
Hossein~Azari Soufiani, David~C Parkes, and Lirong Xia.
\newblock Preference elicitation for general random utility models.
\newblock In \emph{Uncertainty in Artificial Intelligence}, page 596. Citeseer,
  2013.

\bibitem[Sui et~al.(2017)Sui, Zhuang, Burdick, and Yue]{Sui+17}
Yanan Sui, Vincent Zhuang, Joel Burdick, and Yisong Yue.
\newblock Multi-dueling bandits with dependent arms.
\newblock In \emph{Conference on Uncertainty in Artificial Intelligence},
  UAI'17, 2017.

\bibitem[Sui et~al.(2018)Sui, Zoghi, Hofmann, and Yue]{sui18survey}
Yanan Sui, Masrour Zoghi, Katja Hofmann, and Yisong Yue.
\newblock Advancements in dueling bandits.
\newblock In \emph{IJCAI}, pages 5502--5510, 2018.

\bibitem[Suk and Kpotufe(2022)]{Suk22}
Joe Suk and Samory Kpotufe.
\newblock Tracking most significant arm switches in bandits.
\newblock In \emph{Conference on Learning Theory}, pages 2160--2182. PMLR,
  2022.

\bibitem[Thompson(1933)]{thompson1933likelihood}
William~R Thompson.
\newblock On the likelihood that one unknown probability exceeds another in
  view of the evidence of two samples.
\newblock \emph{Biometrika}, 25\penalty0 (3-4):\penalty0 285--294, 1933.

\bibitem[Wu and Liu(2016)]{DTS}
Huasen Wu and Xin Liu.
\newblock Double {T}hompson sampling for dueling bandits.
\newblock In \emph{Advances in Neural Information Processing Systems}, pages
  649--657, 2016.

\bibitem[Wu et~al.(2018)Wu, Iyer, and Wang]{wu+18}
Qingyun Wu, Naveen Iyer, and Hongning Wang.
\newblock Learning contextual bandits in a non-stationary environment.
\newblock \emph{In Proceedings of the 41st International ACM SIGIR Conference
  on Research \& Development in Information Retrieval}, pages 495--504, 2018.

\bibitem[Yue and Joachims(2009)]{Yue+09}
Yisong Yue and Thorsten Joachims.
\newblock Interactively optimizing information retrieval systems as a dueling
  bandits problem.
\newblock In \emph{Proceedings of the 26th Annual International Conference on
  Machine Learning}, pages 1201--1208. ACM, 2009.

\bibitem[Yue and Joachims(2011)]{BTM}
Yisong Yue and Thorsten Joachims.
\newblock Beat the mean bandit.
\newblock In \emph{Proceedings of the 28th International Conference on Machine
  Learning (ICML-11)}, pages 241--248, 2011.

\bibitem[Yue et~al.(2012)Yue, Broder, Kleinberg, and Joachims]{Yue+12}
Yisong Yue, Josef Broder, Robert Kleinberg, and Thorsten Joachims.
\newblock The $k$-armed dueling bandits problem.
\newblock \emph{Journal of Computer and System Sciences}, 78\penalty0
  (5):\penalty0 1538--1556, 2012.

\bibitem[Zoghi et~al.(2014{\natexlab{a}})Zoghi, Whiteson, Munos, Rijke,
  et~al.]{Zoghi+14RUCB}
Masrour Zoghi, Shimon Whiteson, Remi Munos, Maarten~de Rijke, et~al.
\newblock Relative upper confidence bound for the $k$-armed dueling bandit
  problem.
\newblock In \emph{JMLR Workshop and Conference Proceedings}, number~32, pages
  10--18. JMLR, 2014{\natexlab{a}}.

\bibitem[Zoghi et~al.(2014{\natexlab{b}})Zoghi, Whiteson, De~Rijke, and
  Munos]{Zoghi+14RCS}
Masrour Zoghi, Shimon~A Whiteson, Maarten De~Rijke, and Remi Munos.
\newblock Relative confidence sampling for efficient on-line ranker evaluation.
\newblock In \emph{Proceedings of the 7th ACM international conference on Web
  search and data mining}, pages 73--82. ACM, 2014{\natexlab{b}}.

\bibitem[Zoghi et~al.(2015)Zoghi, Karnin, Whiteson, and De~Rijke]{Zoghi+15}
Masrour Zoghi, Zohar~S Karnin, Shimon Whiteson, and Maarten De~Rijke.
\newblock Copeland dueling bandits.
\newblock In \emph{Advances in Neural Information Processing Systems}, pages
  307--315, 2015.

\end{thebibliography}

\onecolumn

\appendix 

\newpage 

\doparttoc
\faketableofcontents
\parttoc 

\onecolumn

\section*{\centering \bfseries{\fontsize{17}{11}\selectfont Supplementary: \\[0.3em] \papertitle}}
\vspace*{1cm}


\section*{Table of Contents}
\vspace{-1.55cm}

\addcontentsline{toc}{section}{Table of Contents}
\part{ } 
\parttoc 

\vspace{0.5cm}

\section*{Notation}
\begin{table}[h]
    \begin{tabularx}{\textwidth}{p{0.2\textwidth}X}
    \toprule 
    $ a_t, b_t $        & Arms selected by the algorithm in round $t$ \\ 
    $ a, a', b$         & Generic fixed arms in $[K]$ \\
    $\delta_t(a, b)$    & Gap between arm $a$ and arm $b$ \\
    $\hdelta_t(a, b)$   & Importance weighted gap estimate \\
      $a_t^*$            & Condorcet winner in round $t$ \\
      $t_\ell$          & First round in the $\ell$-th episode \\ 
      $t_\ell^a$        & Round in the $\ell$-th episode in which $a$ is eliminated from $\mA$ \\
      $\Sw$             & Number of Condorcet Winner Switches \\
     $\tau_1, \dots, \tau_{\Sw}$ & Rounds in which the Condorcet winner changes \\ 
    $a_i^*$             & Condorcet winner in phase $i\in [\Sw]$, i.e.\ $a_t^* = a_i^*$ for $t \in [\tau_i, \tau_{i+1})$ \\
     $\sigS$            & Number of Significant Condorcet Winner Switches \\
     $\htau_1, \dots, \htau_{\sigS}$    & Rounds of Significant CW Switches\\
     $\asafe_i$     & Last safe arm in phase $[\htau_i, \htau_{i+1})$, i.e.\ last arm to satisfy \eqref{eq:sig_shift} \\
     $\Vcw$     & Condorcet Winner Variation \\
     
      \bottomrule
     \end{tabularx}
    \end{table}

\vfill

\newpage 

\section{Proof of Theorem~\ref{thm:main_result}}\label{sec:appendix_proof_thm}
We organize the proof of Theorem~\ref{thm:main_result} as follows.  Section~\ref{subsection:proof_preliminaries} contains basic preliminary facts that will be the foundation of the upcoming proof. Section~\ref{subsection:bounding(a)} then bounds the regret any fixed arm suffers within each episode \emph{before} being eliminated from the good set. Complementary to this, Section~\ref{subsection:bounding(b)} then deals with the regret an arm suffers \emph{after} being eliminated.  

\subsection{Preliminaries}\label{subsection:proof_preliminaries}
In this preliminary section, we introduce a concentration bound on the sum of our estimates $\hdelta_t$ in Section~\ref{subsubsection:concentration}. We then show in Section~\ref{subsubsection:episodes_arms} that the beginning of a new episode implies that the Condorcet winner has changed (on the concentration event), which will be useful later. Finally, Section~\ref{subsubsection:decomposing} decomposes the regret in terms episodes, arms, and rounds, which will form the basis of our analysis.

\subsubsection{Martingale Concentration Bound}\label{subsubsection:concentration} 
We will rely on a similar martingale tail bound as \cite{beygelzimer2011contextual} and \cite{Suk22}, which is based on a version of Freedman's inequality given below. 
\begin{lem}[Theorem~1 in \cite{beygelzimer2011contextual}]\label{lem:beygelzimmer}
    Let $(X_t)_{t \in \N}$ be a martingale difference sequence w.r.t.\ some filtration $(\cF_t)_{t\in \N_0}$. Assume that is $X_t$ is almost surely uniformly bounded, i.e.\ $X_t \leq R$ a.s.\ for some constant~$R$. Moreover, suppose that $\sum_{s=1}^t \E[X_s^2\mid \cF_{s-1}] \leq V_t$ a.s.\ for some sequence of constants $(V_t)_{t\in \N}$. Then, for any $\delta \in (0,1)$, with probability at least $1-\delta$, we have 
    \begin{equation}
        \sum_{s=1}^t X_s \leq (e-1) \left( \sqrt{V_t \log(1/\delta)} + R \log(1/\delta) \right) . 
    \end{equation}
\end{lem}
\begin{proof}
    See Theorem~1 in \cite{beygelzimer2011contextual} and Lemma~1 in \cite{Suk22}. 
\end{proof}

We now apply the above concentration bound to the martingale difference sequence $\hdelta_t(a, b) - \E [\hdelta_t (a, b) \mid \cF_{t-1} ]$.  
\begin{lem}\label{prop:concentration_bound}
    Let $\cE$ be the event that for all rounds $s_1 < s_2$ and all arms $a, b \in [K]$:
    \begin{equation}\label{eq:concentration_bound}
        \abs{\sum_{t=s_1}^{s_2} \hdelta_t(a, b) - \sum_{t=s_1}^{s_2} \E \left[ \hdelta_t(a, b) \mid \cF_{t-1}\right]} \leq c_1 \log(T) \left( K\sqrt{(s_2-s_1)} + K^2\right)
    \end{equation}
    for an appropriately large constant $c_1 > 0$ and where $\cF = \{ \cF_t\}_{t \in \N_0}$ is the canonical filtration generated by observations in past rounds. Then, event $\cE$ occurs with probability at least $1-1/T^2$. 
\end{lem} 
\begin{proof}
    Note that $\hdelta_t(a, b) - \E [\hdelta_t (a, b) \mid \cF_{t-1} ]$ is naturally a martingale difference, since $\E\big[ \hdelta_t(a, b) - \E[\hdelta_t(a, b)\mid \cF_{t-1} ] \mid \cF_{t-1}\big] = 0$ a.s. Using that $|\cA_t| \leq K$, we have that $X_t \leq 2 K^2$ a.s.\ for all rounds $t$. Moreover, we get that 
    \begin{align*}
        \sum_{t=s_1}^{s_2} \E \left[ \hdelta_t^2(a, b) \mid \cF_{t-1}\right] \leq \sum_{t=s_1}^{s_2} \abs{\cA_t}^4 \E \left[ \1_{\{a_t = a, b_t = b\}} \mid \cF_{t-1} \right] = \sum_{t=s_1}^{s_2} \abs{\cA_t}^2 \leq K^2 (s_2 - s_1). 
    \end{align*}
    We can thus apply Lemma~\ref{lem:beygelzimmer} with $R = K^2$ and $V_t = 2 K^2 t$. Using $|x-y| \leq |x|+|y|$ and taking union bounds over $a, b$ and $s_1, s_2$, we then obtain Lemma~\ref{prop:concentration_bound}. 
\end{proof}

\subsubsection{Episodes and Condorcet Winner Switches}\label{subsubsection:episodes_arms}

\begin{lem}\label{lem:phase_intersecting_episodes}
    On event $\cE$, for each episode $[t_\ell, t_{\ell+1})$ with $t_{\ell+1} \leq T$, there exists a change of the CW $\tau_i \in [t_\ell, t_{\ell+1})$. 
\end{lem}
This implies that any phase $[\tau_i, \tau_{i+1})$ will intersect with at most two episodes. 
\begin{proof}
    The start of a new episode means that every arm $a\in [K]$ has been eliminated from $\mA$ at some round in $t_\ell^a \in [t_\ell, t_{\ell+1})$. As a result, there must exist an interval $[s_1, s_2] \subseteq [t_\ell, t_\ell^a)$ and some arm $a' \in[K]$ so that the elimination rule \eqref{eq:elim} holds. Using Lemma~\ref{prop:concentration_bound}, we then find that for some constant $c_2 > 0$: 
    \begin{align}\label{eq:large_regret}
        \sum_{t=s_1}^{s_2} \E \left[ \hdelta_t(a', a) \mid \cF_{t-1} \right] > c_2 \log (T) K \sqrt{(s_2 -s_1) \vee K^2}. 
    \end{align}
    Note that by construction of $\hdelta_t(a', a)$, we always have $\delta_t(a', a) \geq \E[ \hat{\delta}_t(a', a) \mid \cF_{t-1}]$ since
    \begin{equation}\label{eq:hdelta_leq_delta}
        \E[  \hat{\delta}_t(a', a) \mid \cF_{t-1} ] = \begin{cases}
			\delta_t(a', a) & a', a \in \cA_t\\
			-1/2 & \text{otherwise}. 
		\end{cases}
    \end{equation}
    Thus, in view of inequality~\eqref{eq:large_regret}, there exists no arm $a\in [K]$ such that $\max_{a'} \delta_t (a', a) = 0$ for all $t \in [t_\ell, t_{\ell+1})$, i.e.\ no fixed arm is optimal throughout the episode and there must have been a change of Condorcet winner.  
\end{proof}

\subsubsection{Decomposing Regret across Episodes and Arms}\label{subsubsection:decomposing}
We will bound regret of the algorithm withing each episode separately, i.e.\ we consider 
\begin{equation}
    \E\left[ \sum_{t=t_\ell}^{t_{\ell+1}-1} \frac{\delta_t(a_t^*, a_t)  + \delta_t(a_t^*, b_t)}{2} \right],
\end{equation}
where $t_\ell$ is the first round in episode $\ell$ and $a_t^*$ is the Condorcet winner in round $t\in [T]$.

Recall that, every round $t\in [T]$, the algorithm selects an arm $a$ uniformly at random from the active set $\cA_t$. It will then be useful to rewrite \eqref{eq:regret_within_eps} in terms of fixed arms $a\in [K]$. 

\begin{lem}\label{lem:substituting_pi}
    We can write \eqref{eq:regret_within_eps} in terms of the regret suffered by fixed arms: 
    \begin{equation}
        \E\left[ \sum_{t=t_\ell}^{t_{\ell+1}-1} \frac{\delta_t(a_t^*, a_t)  + \delta_t(a_t^*, b_t)}{2} \right] 
        = \E \left[\sum_{a =1}^{K} \sum_{t=t_\ell}^{t_{\ell+1}} \frac{\delta_t (a_t^*, a)}{|\cA_t|} \1_{\{a \in \cA_t\}} \right]
    \end{equation}
\end{lem}
\begin{proof}
As the algorithm independently and symmetrically selects two arms $(a_t, b_t)$ in each round (Line~\ref{line:sample_pair} in Algorithm~\ref{alg:replay}), we can focus on bounding regret for one of the two arms, say $a_t$, by writing 
\begin{equation}\label{eq:regret_within_eps}
     \E\left[ \sum_{t=t_\ell}^{t_{\ell+1}-1} \frac{\delta_t(a_t^*, a_t)  + \delta_t(a_t^*, b_t)}{2} \right] = \E\left[ \sum_{t=t_\ell}^{t_{\ell+1}-1} \delta_t(a_t^*, a_t)\right]. 
\end{equation}

\noindent Conditioning on $t_\ell$ and using the tower property, we then further find that 
\begin{align*}
    \E \left[ \sum_{t=t_\ell}^{t_{\ell+1}} \delta_t (a_t^*, a_t) \right] & = \E \left[ \E \left[ \sum_{t=t_\ell}^{t_{\ell+1}} \delta_t (a_t^*, a_t) \mid t_\ell \right]\right] \\
    & = \E \left[ \sum_{t=t_\ell}^{T}  \E \left[ \1_{\{ t < t_{\ell+1}\}} \E\left[ \delta_t (a_t^*, a_t) \mid \cF_{t-1}\right] \mid t_\ell \right] \right] \\ 
    & = \E \left[ \sum_{t=t_\ell}^{T} \sum_{a\in \cA_t} \E \left[ \1_{\{ t < t_{\ell+1}\}}  \mid t_\ell \right] \frac{\delta_t (a_t^*, a)}{|\cA_t|} \right] = \E \left[ \sum_{t=t_\ell}^{t_{\ell+1}} \sum_{a \in \cA_t} \frac{\delta_t (a_t^*, a)}{|\cA_t|} \right] ,
\end{align*}
where we used that $\1_{\{ t < t_{\ell+1}\}}$ is $\cF_{t-1}$-measurable and
$$\E\left[\delta_t(a_t^*, a_t) \mid \cF_{t-1}\right] = \sum_{a\in \cA_t} \frac{\delta_t(a_t^*, a)}{|\cA_t|}.$$
Lastly, Lemma~\ref{lem:substituting_pi} then follows from rewriting the sum over $a \in \cA_t$ using the indicator $\1_{\{a \in \cA_t\}}$ and swapping the order of the sums. 
\end{proof}
\noindent 
In an important next step, we split the dynamic regret for \emph{each fixed arm $a\in [K]$} into: 
\begin{enumerate}
    \item [(i)] the regret we suffer from playing arm $a$ in the $\ell$-th episode before its elimination from $\mA$, 
    \item[(ii)] the regret we suffer from (re)playing arm $a$ in the $\ell$-th episode after its elimination from $\mA$. 
\end{enumerate}
Recall that $t_\ell^a\in[t_\ell, t_{\ell+1})$ denotes the time that arm $a$ is eliminated from $\mA$ in episode $\ell$. Using Lemma~\ref{lem:substituting_pi}, we then decompose the dynamic regret in episode $\ell$ as
\begin{equation}
    E\left[ \sum_{t=t_\ell}^{t_{\ell+1}-1} \frac{\delta_t(a_t^*, a_t)  + \delta_t(a_t^*, b_t)}{2} \right] =  
    \underbrace{\E \left[  \sum_{a=1}^K \sum_{t=t_\ell}^{t_{\ell}^a-1} \frac{\delta_t (a_t^*, a)}{|\cA_t|} \right]}_{\terma} + 
    \underbrace{\E \left[  \sum_{a=1}^K \sum_{t=t_\ell^a}^{t_{\ell+1}-1} \frac{\delta_t (a_t^*, a)}{|\cA_t|} \1_{\{a \in \cA_t\}} \right]}_{\termb}, 
\end{equation}
where for $\terma$ we used that $a \in \mA$ implies $a \in \cA_t$ by construction of these sets. 
For every fixed arm, $\terma$ corresponds to the regret suffered before said arm is eliminated from the master set. Accordingly, $\termb$ is the regret due to replaying an arm after its elimination from the master set. 
The remainder of the proof is mainly concerned with bounding $\terma$ and $\termb$ appropriately. 

\subsection{Bounding $\terma$: Regret Before Elimination}\label{subsection:bounding(a)}
We begin by assuming w.l.o.g.\ that $t_\ell^1 \leq \cdots \leq t_\ell^K$ so that for each round $t < t_\ell^a$ all arms $a' \geq a$ are element in $\mA \subseteq \cA_t$. As a result, we have $\abs{\cA_t} \geq K+1-a$ for all $t \leq t_\ell^a$, and thus 
\begin{align}\label{eq:swap_At_for_K}
    \E \left[  \sum_{a=1}^K \sum_{t=t_\ell}^{t_{\ell}^a-1} \frac{\delta_t (a_t^*, a)}{|\cA_t|} \right] \leq \E \left[  \sum_{a=1}^K \sum_{t=t_\ell}^{t_{\ell}^a-1} \frac{\delta_t (a_t^*, a)}{K+1-a} \right]. 
\end{align}
As we can see, the denominator will eventually account for a factor of $\log(K) \approx \sum_{a=1}^K 1/a$. We now concentrate on bounding the inner sum in \eqref{eq:swap_At_for_K}, i.e.\ the regret of any fixed arm before being eliminated in the $\ell$-th episode.  


\subsubsection{Bounding $\E[\sum_{t=t_\ell}^{t_\ell^a -1} \delta_t(a_t^*, a)]$ for any fixed arm $a\in [K]$}
This section is devoted to proving the following upper bound.
\begin{lem}\label{lem:bound_for_(a)}
For some constant $c>0$:
\begin{equation}
    \E \left[\sum_{t=t_\ell}^{t_\ell^a-1} \delta_t(a_t^*, a)\right] \leq c \log^2(T) K \, \E \left[\sum_{i \in \mathrm{Phases}(t_\ell, t_\ell^a)} \sqrt{\tau_{i+1}- \tau_i}\right] + \frac{K}{T^2} + \frac{1}{T} . 
\end{equation}
\end{lem}
To prove Lemma~\ref{lem:bound_for_(a)}, we will divide the interval $[t_\ell, t_\ell^a)$ into segments over the course of which arm $a$ suffers large regret and show that not too many of such segments will occur in interval $[t_\ell, t_\ell^a)$, i.e.\ until arm $a$ is being eliminated from $\mA$. The definition of such bad segments is analogous to their construction in \cite{Abasi2022A} and \cite{Suk22}. Whereas prior work utilizes such segments to bound the regret of the last arm considered good in an episode, i.e.\ the last arm in $\mA$, we will instead derive a regret bound for \emph{any fixed arm $a$}. While the according regret bound will be in some sense weaker, it will still be sufficiently tight for our purposes. 
We here follow the notation in~\cite{Suk22}. 

\begin{defn}[Bad Segments]\label{def:bad_segments}
Fix $t_\ell$ and let $[\tau_i, \tau_{i+1})$ be any phase intersecting $[t_\ell, T)$. For an arm $a$, define rounds $s_{i,j}(a) \in [t_\ell \vee \tau_i, \tau_{i+1})$ recursively as follows: let $s_{i,0}(a) = t_\ell \vee \tau_i$ and define $s_{i,j+1} (a)$ as the smallest round in $(s_{i, j}(a), \tau_{i+1})$ such that arm $a$ satisfies for some constant $c_3 > 0$: 
\begin{equation}\label{eq:bad_segment}
    \sum_{t={s_{i, j}}(a)}^{s_{i,j+1}(a)} \delta_t(a_i^*, a) > c_3 \log(T) K \sqrt{s_{i, j+1} (a) - s_{i, j}(a)},
\end{equation}
if such round $s_{i,j+1}(a)$ exists. Otherwise, we let $s_{i,j+1}(a) = \tau_{i+1} - 1$. We refer to the intervals $[s_{i,j}, s_{i, j+1})$ as bad segments if \eqref{eq:bad_segment} is satisfied. If a segment does not satisfy~\eqref{eq:bad_segment}, we refer to them as non-bad segments.\footnote{Note that by definition every segment but the last segment in a given phase must always satisfy \eqref{eq:bad_segment}}  
\end{defn}

Note that the concept of bad segments will become useful later as, for a fixed $t_\ell$, by definition of the bad segments, we can always upper bound the dynamic regret on an interval $[s_{i,j}(a), s_{i,j+1}(a))$ by 
\begin{equation}\label{eq:bound_regret_using_segments}
    \sum_{t=s_{i,j}(a)}^{s_{i,j+1}(a)-1} \delta_t (a_t^*, a) \leq c_3 \log(T) K \sqrt{s_{i,j+1}(a) - s_{i,j}(a)}. 
\end{equation}
We now define the \emph{bad round} for an arm $a$ as the smallest round when the aggregated regret of bad segments exceeds $\sqrt{\text{interval length}}$ regret. 

\begin{defn}[Bad Round]\label{def:bad_round}
    Fix $t_\ell$ and some arm $a$. The {bad round} $s(a) > t_\ell$ is defined as the smallest round which satisfies for some universally fixed constant $c_4>0$:
\begin{equation}\label{eq:bad_round}
    \sum_{\substack{(i,j)\colon s_{i,j+1}(a) < s(a)}} \sqrt{s_{i,j+1}(a) - s_{i,j}(a)} > c_4 \log(T) \sqrt{s(a) - t_\ell},
\end{equation}
where the sum is over all bad segments with $s_{i,j+1}(a) < s(a)$. 
\end{defn}

For a given episode $\ell$, we will show that arm $a$ is eliminated with high probability by the time the bad round $s(a)$ occurs. To this end, we will introduce perfect replays, i.e.\ those runs of $\base$ which are properly timed and eliminate arm $a$ before it aggregates large regret. 

\subsubsection{Perfect Replays} 
The following result will become very useful and makes the intuition precise that on the concentration event the Condorcet winner will not be eliminated. More precisely, any run of $\base(s, m)$ scheduled in phase $i$ will never eliminate $a_i^*$ inside phase $i$ as long as our concentration bound holds. 
\begin{lem}\label{lemma:no_eviction_inside_replay}
    On event $\cE$, no run of $\base(s, m)$ with $s\in [\tau_i, \tau_{i+1})$ ever eliminates arm $a_i^*$ before round $\tau_{i+1}$. 
\end{lem}
\begin{proof}
Suppose the contrary that some $\base(s, m)$ with $s \in [\tau_i, \tau_{i+1})$ eliminates arm $a_i^*$ before round $\tau_{i+1}$. Then, we must have for some arm $a\in[K]$ and interval $[s_1, s_2] \subseteq [s, \tau_{i+1})$ that 
\begin{equation}
    C \log(T) K\sqrt{(s_2-s_1) \vee K^2} < \sum_{t=s_1}^{s_2} \hdelta_t (a, a_i^*),
\end{equation}
which using the concentration bound \eqref{eq:concentration_bound} implies on event $\cE$ that 
\begin{equation}
    c_2 \log(T) K \sqrt{(s_2-s_1) \vee K^2} < \sum_{t=s_1}^{s_2} \E \left[ \hdelta_t(a, a_i^*) \mid \cF_{t-1} \right] \leq \sum_{t=s_1}^{s_2} \delta_t(a, a_i^*)  , 
\end{equation}
where the last inequality holds by merit of \eqref{eq:hdelta_leq_delta}. Now, by the definition of arm $a_i^*$ as the Condorcet winner in phase $i$, we must have $\delta_t(a, a_i^*) \leq 0$ for all $t \in [\tau_i, \tau_{i+1})$ and all $a \in [K]$.  Lemma~\ref{lemma:no_eviction_inside_replay} then follows from contradiction. 
\end{proof}

This leads to the following important property of $\base$ that states that properly timed replays of sufficient length will eliminate arms from $\mA$ in the course of their bad segments. We call such calls of $\base$ \emph{perfect replays}. 

\begin{prop}[Perfect Replay]\label{prop:perfect_replay}
    Suppose that event $\cE$ holds. Let $[s_{i,j}(a), s_{i, j+1}(a))$ be a bad segment w.r.t.\ arm $a$ and let $\tilde s_{i,j}(a) = \big\lceil \frac{s_{i,j}(a)+ s_{i,j+1}(a)}{2}\big\rceil$ be the midpoint of the interval. It holds that any run of $\base(s,m)$ with $s \in [s_{i,j}(a), \tilde s_{i,j}(a)]$ and $m \geq s_{i,j+1}(a) - s_{i,j}(a)$ will eliminate arm $a$ from $\mA$. We refer to such calls of $\base$ as \emph{perfect replays} w.r.t.\ arm $a$. 
\end{prop}
\begin{proof} 
Let $\base(s, m)$ be a replay such that $s \in [s_{i,j}(a), \tilde s_{i,j}(a)]$ and $m \geq s_{i,j+1}(a) - s_{i,j}(a)$. 
As any bad segment is by definition contained inside a phase, Lemma~\ref{lemma:no_eviction_inside_replay} tells us that $a_i^* \in \cA_t$ for all $t \in [\tilde s_{i,j}(a), s_{i, j+1}(a)]$. Recall that the estimates $\hdelta_t(a_i^*, a)$ are unbiased if $a, a_i^*\in\cA_t$ and we are thus able to obtain unbiased estimates of $\sum_{t=\tilde s_{i,j}(a)}^{s_{i, j+1}(a)} \delta_t (a_i^*, a)$. What is left to show is that in fact arm $a$ suffers sufficiently large regret to cause its elimination on this interval. To this end, by definition of the bad segments and basic algebraic manipulation, we find that
\begin{align*}
    \sum_{t = \tilde s_{i,j}(a)}^{s_{i,j+1}(a)} \delta_{t}(a_i^*,a) & =
	\sum_{t = s_{i,j}(a)}^{s_{i,j+1}(a)}  \delta_{t}(a_i^*,a)
	- \sum_{t=s_{i,j}(a)}^{\tilde s_{i,j}(a)-1} \delta_t(a_i^*,a)\\
	& \overset{\eqref{eq:bad_segment}}{\geq} c_3 \log(T)K\left(\sqrt{s_{i,j+1}(a) - s_{i,j}(a)} - \sqrt{\tilde s_{i,j}(a) -1 - s_{i,j}(a)}\right) \\[1em]
	&\geq \frac{c_3}{4}\log(T) K \sqrt{s_{i,j+1}(a) - \tilde s_{i,j}(a)}. 
\end{align*}
Using that $\sum_{t=\tilde s_{i,j}(a)}^{s_{i, j+1}(a)} \hdelta_t (a_i^*, a)$ is an unbiased estimate of $\sum_{t=\tilde s_{i,j}(a)}^{s_{i, j+1}(a)} \delta_t (a_i^*, a)$ and applying the concentration bound \eqref{eq:concentration_bound}, this shows that arm $a$ satisfies the elimination rule \eqref{eq:elim} over interval $[\tilde s_{i,j}(a), s_{i,j+1}(a)]$ and will thus be eliminated by $\base(s,m)$. 
\end{proof}

\subsubsection{Perfect replays are scheduled w.h.p.}
Following \cite{Suk22}, we will now show that a perfect replay that eliminates arm $a$ is scheduled before round $s(a)$ with high probability. 
A replay $\base(s,m)$ is scheduled if $B_{s,m} = 1$ and the random variables $B_{s,m}$ with $s \geq t_\ell$ are conditionally independent on $t_\ell$ (see Line~\ref{line:sample_replay_schedule} in Algorithm~\ref{alg:anaconda}). 
We are thus interested in perfect replays $\base(s,m)$ such that for any bad segment $[s_{i,j}(a), s_{i,j+1}(a))$ with $s_{i,j+1}(a) < s(a)$, we have $s\in [s_{i,j}(a), \tilde s_{i,j}(a)] $ and $m \geq s_{i, j+1}(a) -s_{i,j}(a)$. 
Moreover, we define $m_{i,j}$ as the smallest element in $\{2, \dots, 2^{\lceil \log(T)\rceil}\}$ such that $m_{i,j} \geq s_{i, j+1}(a) -s_{i,j}(a)$, which implies that $s_{i, j+1}(a) -s_{i,j}(a) \geq \frac{m_{i,j}}{2}$.  
We will obtain the high probability guarantee via concentration on the sum
\begin{equation}
    X({t_\ell, s(a)}) = \sum_{\substack{(i,j) \colon s_{i,j+1}(a) < s(a) }} \, \sum_{s=s_{i,j}(a)}^{\tilde s_{i,j}(a)}
    B_{s, m_{i,j}} .
\end{equation}

\begin{lem}\label{lem:replay_scheduled_whp}
    Let $\cE'(t_\ell)$ denote the event that $X({t_\ell, s(a)}) \geq 1$ for all arms $a$, i.e.\ a perfect replay is scheduled before round $s(a)$. We have $\P(\cE'(t_\ell) \mid t_\ell) \geq 1-K/T^3$. 
\end{lem}
\begin{proof}
    Recalling that $B_{s,m} \mid t_\ell \sim \text{Bernoulli}\left(\frac{1}{\sqrt{m (s-t_{\ell})}}\right)$, we find that 
    \begin{align*}
        \E[X({t_\ell, s(a)}) \mid t_\ell] 
        & \geq \frac{1}{\sqrt{2}}\sum_{\substack{(i,j) \colon \\ s_{i,j+1}(a) < s(a) }} \frac{\tilde s_{i,j}(a) - s_{i,j}(a)}{\sqrt{s_{i,j+1}(a) -s_{i,j}(a)}\sqrt{\rule{0pt}{1.8ex} s(a)-t_\ell}} \\ 
        & \geq \frac{1}{4} \sum_{\substack{(i,j) \colon \\ s_{i,j+1}(a) < s(a) }} \sqrt{\frac{s_{i,j+1}(a) - s_{i,j}(a)}{s(a) -t_\ell}} \ 
        \overset{\eqref{eq:bad_round}}{\geq} \frac{c_4}{4} \log(T) 
    \end{align*}
    For $c_4$ sufficiently large the standard Chernoff bound tells us that 
    \begin{align*}
        \P\left(X({t_\ell, s(a)}) \leq \frac{\E[X({t_\ell, s(a)}) \mid t_\ell]}{2} \mid t_\ell\right) \leq \exp\left(-\frac{\E[X({t_\ell, s(a)}) \mid t_\ell]}{8}\right) \leq \frac{1}{T^{3}}.  
    \end{align*}
    The desired bound then follows from taking a union bound over all arms in $[K]$. 
\end{proof}
Now, on event $\cE \cap \cE'(t_\ell)$, it must hold that $t_{\ell}^a < s(a)$ for all arms $a\in [K]$, since otherwise $a$ would have been eliminated by some perfect replay before round $t_\ell^a$ (by definition of event $\cE'(t_\ell)$). As the bad round $s(a)$ is defined as the \emph{smallest} round satisfying \eqref{eq:bad_round}, we then have
\begin{equation}\label{eq:aggregated_bad_seg_bound}
    \sum_{\substack{(i,j) \colon s_{i,j+1}(a) < t_\ell^a}} \sqrt{s_{i,j+1}(a) - s_{i,j} (a)} \leq c_4 \log(T) K \sqrt{t_{\ell}^a- t_\ell}.  
\end{equation}
Hence, in view of equation~\eqref{eq:bound_regret_using_segments}, over the bad segments, the regret of arm $a$ is at most of order $\log^2(T) \sqrt{t_\ell^a - t_\ell}$.  
Moreover, for every last segment in some phase $i$, $[s_{i,j}, s_{i,j+1}(a))$, as well as the final segment $[s_{i,j}(a), t_{\ell}^a)$, we know that the regret suffered from playing $a$ is upper bounded by $c_3 \log(T) \sqrt{\tau_{i+1}-\tau_i}$ by definition of non-bad segments (Definition~\ref{def:bad_segments}). Therefore, on event $\cE \cap \cE'(t_\ell)$, it follows from equation \eqref{eq:aggregated_bad_seg_bound} and the above that 
\begin{equation}
    \sum_{t=t_\ell}^{t_\ell^a-1} \delta_t(a_t^*, a) \leq c_5 K \log^2(T) \sum_{i \in \phases(t_\ell, t_\ell^a)} \sqrt{\tau_{i+1}- \tau_i},
\end{equation}
where we used that $\sqrt{t_\ell^a- t_\ell} \leq \sum_{i \in \phases(t_\ell, t_\ell^a)} \sqrt{\tau_{i+1}-\tau_i}$. Finally, we obtain Lemma~\ref{lem:bound_for_(a)} by taking expectation and using that $\cE \cap \cE'(t_\ell)$ holds with high probability, 
\begin{align*}
    \E \left[ \sum_{t=t_\ell}^{t_\ell^a-1} \delta_t(a_t^*, a) \right] 
    &\leq \E\left[ \Bigg[ \1_{\{\cE \cap \cE'(t_\ell)\}} \sum_{t=t_\ell}^{t_\ell^a-1} \delta_t(a_t^*, a)  \mid t_\ell \Bigg] \right] + T \big( \P(\cE^c) + \P(\cE'(t_\ell)^c \mid t_\ell) \big)  \\ 
    & \leq c_5 K \log^2(T) \,  \E \left[ \sum_{i \in \phases(t_\ell, t_\ell^a)} \sqrt{\tau_{i+1}- \tau_i} \right] + \frac{1}{T} + \frac{K}{T^2}.  
\end{align*}

\subsubsection{Summing Over Arms}
Note that $t_\ell^a \leq t_{\ell+1}$ for all $a \in [K]$ by definition of $t_\ell^a$. Then, summing over all arms, it follows from Lemma~\ref{lem:bound_for_(a)} and \eqref{eq:swap_At_for_K} that for some constant $c_6 >0$:
\begin{align}\label{eq:final_bound_(a)}
    \E \left[  \sum_{a=1}^K \sum_{t=t_\ell}^{t_{\ell}^a-1} \frac{\delta_t (a_t^*, a)}{|\cA_t|} \right] \leq c_6 K \log^3(T) \,   \E \left[ \sum_{i \in \phases(t_\ell, t_{\ell+1})} \sqrt{\tau_{i+1}- \tau_i} \right], 
\end{align}
where we loosely upper bound $\log(K)$ by $\log(T)$. 

\subsection{Bounding $\termb$: Regret After Elimination}\label{subsection:bounding(b)}

Before we can begin, we will have to lay some groundwork to simplify the analysis in later steps. 
Recall the definition of bad segments from Section~\ref{subsection:bounding(a)} and define for every phase $[\tau_{i}, \tau_{i+1})$ intersecting with $[t_\ell^a, t_{\ell+1})$, i.e.\ $i \in \phases(t_\ell^a, t_{\ell+1})$, the segments $[s_{i,j}(a), s_{i, j+1})$ as in Definition~\ref{def:bad_segments}. 

We will split the regret due to bad segments, i.e.\ those that satisfy \eqref{eq:bad_segment}, from the regret due to non-bad segments, i.e.\ the last segments in a phase that do no satisfy \eqref{eq:bad_segment}. For a fixed arm $a \in [K]$, we let $\badseg(a)$ denote the rounds $t \in [t_\ell, t_{\ell+1})$ such that $t\in [s_{i,j}(a), s_{i, j+1}(a))$ for any \emph{bad} segment $[s_{i,j}(a), s_{i,j+1}(a))$. 

By the definition of a non-bad segment (w.r.t.\ arm $a$), we know that that there is at most one such segment in every phase and that the regret of arm $a$ in each segment is upper bounded by $c_3 \log(T) \sqrt{\tau_{i+1}- \tau_i}$, where $[\tau_i, \tau_{i+1})$ is the phase that contains the segment. To take care of the denominator $|\cA_t|$, assume w.l.o.g.\ that there is a run of $\base(t_\ell^a, m)$ that remains active and uninterrupted until the final round $T$.\footnote{Note that this is w.l.o.g.\ when bounding $1/|\cA_t|$ as any interrupting call of $\base$ would only increase $|\cA_t|$ by resetting it to $[K]$.} We can then reorder arms $a\in [K]$ according to the round that they are being eliminated by $\base(t_\ell^a, m)$, which gives $|\cA_t| \geq K+1-a$ whenever $a \in \cA_t$. As before, this yields a factor of $\log(K)$ when summing over all arms.  
We then bound $\termb$ over non-bad segments as
\begin{equation}\label{eq:non_bad_segments_replays}
    \E \left[  \sum_{a=1}^K \sum_{t=t_\ell^a}^{t_{\ell+1}-1} \frac{\delta_t (a_t^*, a)}{|\cA_t|} \1_{\{a \in \cA_t, t \not \in \badseg(a)\}} \right] \leq c_3 K \log(K) \log(T) \E \left[\sum_{i\in \phases(t_\ell^a, t_{\ell+1})} \sqrt{\tau_{i+1} - \tau_i}\right].
\end{equation}

The more challenging task is now to bound $\termb$ for rounds in bad segments.  
Recall that, for a fixed arm $a\in [K]$, the sum in question relates to the expected regret suffered within an episode from replaying arm $a$ after it has been eliminated from $\mA$, i.e.\ after time $t_\ell^a$. We begin by a straightforward upper bound. To this end, for a given replay $\base(s, m)$, let $M(s, m, a)$ be the last round in $[s, s+m]$, where arm $a$ is active in $\base(s, m)$ and all of its children. Then, 
\begin{equation}\label{eq:upper_bound_sum_over_replays}
    \E \left[  \sum_{a=1}^K \sum_{t=t_\ell^a}^{t_{\ell+1}-1} \frac{\delta_t (a_t^*, a)}{|\cA_t|} \1_{\{a \in \cA_t, t \in \badseg(a)\}} \right] \leq \E \left[ \sum_{a=1}^K \sum_{s=t_\ell+1}^{t_{\ell+1}-1} \sum_{m} \1_{\{ B_{s, m} = 1\}} \sum_{t=s \vee t_\ell^a}^{M(s, m, a)} \frac{\delta_t (a_t^*, a)}{\abs{\cA_t}} \1_{\{ t \in \badseg(a)\}} \right],
\end{equation}
where the most inner sum on the right hand side is for $m \in \{2, \dots, 2^{\lceil \log(T)\rceil}\}$. We will keep the convention that whenever a sum over $m$ is not further specified, it will be over the above set. 
Note that \eqref{eq:upper_bound_sum_over_replays} is a loose upper bound. While of course only a single call of $\base$ can be active at any point in time, we here sum over every possible replay and ignore the potential nesting and interleaving of replays. In particular, this upper bound is justified as each $\delta_t(a_t^*, a)$ is non-negative by definition of the CW $a_t^*$. The looseness of \eqref{eq:upper_bound_sum_over_replays} will pose no obstacle, as the remainder of our upper bounds will be sufficiently tight as we will see.

Again, we first take care of the dependence on $K$ due to the denominator on the right hand side of \eqref{eq:upper_bound_sum_over_replays}. Note that for a fixed $\base(s,m)$ if $a_k$ is the $k$-th arm to be eliminated by $\base(s,m)$, then $\min_{t \in [s, M(s, m, a_k)]} \abs{\cA_t} \geq K+1-k$. Similarly to before, this will result in a multiplicative $\log(K)$ term when eventually switching the order of the sums and summing over all arms. For now, we therefore focus on the expression 
\begin{equation}\label{eq:bounding_b_2}
    \E \left[ \sum_{s=t_\ell+1}^{t_{\ell+1}-1} \sum_{m} \1_{\{ B_{s, m} = 1\}} \sum_{t=s \vee t_\ell^a}^{M(s, m, a)} \delta_t (a_t^*, a) \1_{\{ t \in \badseg(a)\}} \right]
\end{equation}
for any fixed arm $a\in [K]$. To deal with this quantity, it will be helpful to distinguish between two types of replays, i.e.\ calls of $\base$, which we refer to as confined and unconfined replays. 

\definecolor{NewGreen}{RGB}{0, 102, 0}
\definecolor{NewBlue}{RGB}{0, 0, 102}
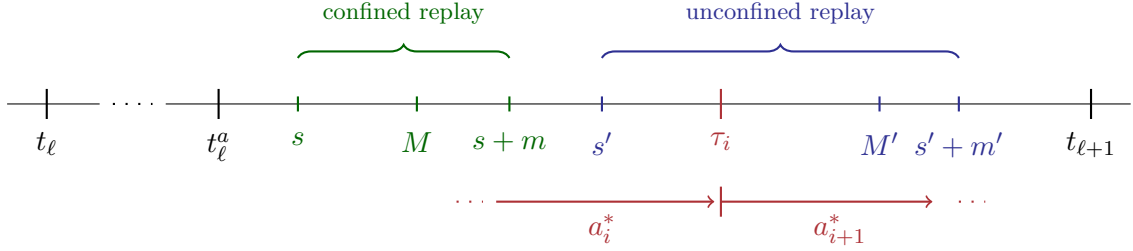
\begin{figure}[t]
    \centering
   \begin{tikzpicture}[x=50]
      \draw (-0.5,0) -- (0.2,0);
      \draw (0.7,0) -- (8,0);
       \draw[loosely dotted, thick] (0.3,0) -- (0.6,0);

       \draw (-0.2, 0) node[below=7pt] {$t_\ell$};
      \draw[thick] (-0.2,-0.2) -- (-0.2,0.2);
       
       \draw (1.1, 0) node[below=7pt] {$t_\ell^a$};       
       \draw[thick] (1.1,-0.2) -- (1.1,0.2);
       
        \draw (7.7, 0) node[below=7pt] {$t_{\ell+1}$};
      \draw[thick] (7.7,-0.2) -- (7.7,0.2);

       \draw (1.7, 0) node[below=7pt] {\color{NewGreen}{$s$}};       
       \draw[NewGreen, thick] (1.7,-0.1) -- (1.7,0.1);
       
       \draw (3.3, 0) node[below=7pt] {\color{NewGreen}$s+m$};       
       \draw[NewGreen, thick] (3.3,-0.1) -- (3.3,0.1);
       
        \draw [thick, NewGreen, decorate, decoration={brace,amplitude=5pt}, xshift=0pt,yshift=6pt](1.7,0.4) -- (3.3,0.4) node[NewGreen,midway, yshift=0.6cm] {\footnotesize{confined replay}};

        \draw (2.6, 0) node[below=7pt] {\color{NewGreen}$M$};       
       \draw[NewGreen, thick] (2.6,-0.1) -- (2.6,0.1);

        \draw (4, 0) node[below=7pt] {\color{Blue}$s'$};       
       \draw[Blue, thick] (4,-0.1) -- (4,0.1);
      
        \draw (6.7, 0) node[below=7pt] {\color{Blue}$s'+m'$};       
       \draw[Blue, thick] (6.7,-0.1) -- (6.7,0.1);
       
        \draw [thick, Blue, decorate, decoration={brace,amplitude=5pt}, xshift=0pt,yshift=6pt](4,0.4) -- (6.7,0.4) node[Blue,midway, yshift=0.6cm] {\footnotesize{unconfined replay}};
       
        \draw (6.1, 0) node[below=7pt] {\color{Blue}$M'$};       
       \draw[Blue, thick] (6.1,-0.1) -- (6.1,0.1);

        \draw (4.9, 0) node[below=7pt] {\color{Maroon}$\tau_i$};       
       \draw[Maroon, thick] (4.9,-0.2) -- (4.9,0.2);

        \draw[Maroon, loosely dotted, thick] (2.9,-1.3) -- (3.1,-1.3);
       \draw[Maroon, thick, ->] (3.2, -1.3) -- (4.85, -1.3);
        \draw (4.0, -1.3) node[below=1pt] {\color{Maroon}$a_i^*$};       
       
        \draw[Maroon, loosely dotted, thick] (6.7,-1.3) -- (6.9,-1.3);
       \draw[Maroon, thick, ->] (4.9, -1.3) -- (6.5, -1.3);
        \draw (5.8, -1.3) node[below=1pt] {\color{Maroon}$a_{i+1}^*$};       
      
       \draw[Maroon, thick] (4.9,-1.5) -- (4.9,-1.1);

   \end{tikzpicture}
    \caption{For some episode $[t_\ell, t_{\ell+1})$ and arm $a \in [K]$, an example of a \textcolor{NewGreen}{confined replay} and a \textcolor{Blue}{unconfined replay}, where \textcolor{NewGreen}{$M=M(s, m, a)$} and \textcolor{Blue}{$M'=M(s', m', a)$}. 
    When a replay $\textcolor{Blue}{\base(s', m')}$ intersects with more than one \textcolor{Maroon}{phase}, the CW in the next phase \textcolor{Maroon}{$[\tau_i, \tau_{i+1})$}, denoted \textcolor{Maroon}{$a_{i+1}^*$}, could be evicted before the beginning of that phase, i.e.\ in the interval $[s', \tau_i)$.}
    \label{fig:replays}
\end{figure}

\begin{defn}[Confined and Unconfined Replays]
    For a fixed $t_\ell$, we call $\base(s, m)$ \emph{confined} if there exists $i \in \phases (t_\ell, T)$ such that $[s, s+m] \subseteq [\tau_i, \tau_{i+1})$, i.e.\ the replay intersects with a single phase only. In turn, we say that $\base(s, m)$ is \emph{unconfined} if for all $i \in \phases (t_\ell, T)$, we have $[s, s+m] \not \subseteq [\tau_i, \tau_{i+1})$.
\end{defn}
An illustration of confined and unconfined replays is given in Figure~\ref{fig:replays}.


We proceed by upper bounding the inner sum $\sum_{t=s \vee t_\ell^a}^{M(s, m, a)} \delta_t (a_t^*, a) \1_{\{ t \in \badseg(a)\}}$ for confined and unconfined replays separately. The bound for confined replays comes with no major intricacies, whereas bounding the regret due to unconfined replays is slightly more involved.

\subsubsection{Bounding Regret for Confined Replays}
We begin by bounding, the inner sum $\sum_{t=s \vee t_\ell^a}^{M(s, m, a)}\delta_t(a_t^*, a)$ for any confined replay in terms of the replay duration~$m$. 
\begin{lem}\label{lem:confined_replay_bound}
On event $\cE$, for any fixed arm $a$ and confined replay $(s, m)$, it holds that
\begin{equation*}
    \sum_{t=s \vee t_\ell^a}^{M(s, m, a)}\delta_t(a_t^*, a) \leq c_2 \log(T) K \sqrt{m}. 
\end{equation*}
\end{lem}
\begin{proof}[Proof of Lemma~\ref{lem:confined_replay_bound}]
    Consider any confined replay $\base(s, m)$ with $[s, s+m] \subseteq [\tau_i, \tau_{i+1})$ for some phase $i$. 
    This implies that on interval $[s, s+m]$ the Condorcet winner remains the same, i.e.\ $a_t^* = a_i^*$ for all $t \in [s, s+m]$. Now, recall from Lemma~\ref{lemma:no_eviction_inside_replay} that, on event $\cE$, arm $a_i^*$ will not be eliminated inside of $[s, s+m]$ as it is a subset of phase $[\tau_i, \tau_{i+1})$. As a result, we must have $a, a_i^* \in \cA_t$ for all $t \in [s\vee t_\ell^a, M(s, m, a)]$ and our estimate $\hdelta_t(a_i^*, a)$ is thus unbiased. Since $M(s, m, a)$ is the last round that arm $a$ is retained by $\base(s, m)$ (and its children), it follows from the elimination rule~\eqref{eq:elim} and the concentration bound~\eqref{eq:concentration_bound} that 
    \begin{align*}
        \sum_{t=s \vee t_\ell^a}^{M(s, m, a)}\delta_t(a_i^*, a) \leq c_2 \log(T) K \sqrt{M(s, m, a) - s \vee t_\ell^a} \leq c_2 \log(T) K \sqrt{m},
    \end{align*} 
    where the last inequality uses that $M(s, m, a) \leq s+m$. 
\end{proof}

\subsubsection{Bounding Regret for Unconfined Replays}

\begin{lem}\label{lem:unconfined_replay_bound}
    On event $\cE \cap \cE''(t_\ell)$, for any fixed arm $a$ and unconfined replay $(s, m)$, it holds that 
    \begin{equation*}
        \sum_{t=s \vee t_\ell^a}^{M(s, m, a)} \delta_t(a_t^*, a) \1_{\{ t \in \badseg(a)\}} \leq c_5 \log^2(T) K \big( \sqrt{s- t_\ell} + 2 \sqrt{m} \big). 
    \end{equation*}
    Here, the event $\cE''(t_\ell)$ is a concentration event similar to that in Lemma~\ref{lem:replay_scheduled_whp} and will be defined in the following. 
\end{lem} 
    
\begin{proof}[Proof of Lemma~\ref{lem:unconfined_replay_bound}]
    Consider any unconfined replay $\base(s, m)$ with $s \in [t_\ell, t_{\ell+1})$. Let $i$ be the phase so that $s \in [\tau_{i-1}, \tau_{i})$. We can then split the sum over $t\in [s\vee t_\ell^a, M(s,m, a)]$ into the rounds before the Condorcet winner changes for the first time within $[s, s+m]$ and the remaining rounds, i.e.\ 
    \begin{equation}\label{eq:regret_until_tau_and_until_M}
        \sum_{t=s \vee t_\ell^a}^{M(s, m, a)} \delta_t (a_t^*, a) = \sum_{t=s \vee t_\ell^a}^{\tau_{i}-1} \delta_t(a_i^*, a) + \sum_{t=\tau_{i}}^{M(s, m, a)} \delta_t (a_{t}^*, a) . 
    \end{equation}
    Note that the interval $[\tau_{i}, M(s, m, a)]$ can itself span over several phases. The first sum on the right hand side can be bounded as in Lemma~\ref{lem:confined_replay_bound}. Using Lemma~\ref{lemma:no_eviction_inside_replay}, the elimination rule, and the concentration bound, we get 
    \begin{equation*}
        \sum_{t=s\vee t_\ell^a}^{\tau_{i}-1} \delta_t(a_i^*, a) \leq c_2 \log(T) K \sqrt{m}.
    \end{equation*}
    The second sum cannot be bounded in a similar way, as we cannot guarantee that the Condorcet winner in some round $t \in [\tau_{i}, M(s, m, a)]$ has not been eliminated in prior rounds $[s\vee t_\ell^a, \tau_{i})$. For example in Figure~\ref{fig:replays}, the unconfined replay $\base(s', m')$ could have eliminated $a_{i+1}^*$ on interval $[s', \tau_i)$ before it became the Condorcet winner. We may therefore fail to detect that $a$ suffers large regret without additional replays. 
    
    To resolve this difficulty, we can reuse part of the arguments from Section~\ref{subsection:bounding(a)}. 
    Define the bad segments $[s_{k,j}(a), s_{k, j+1}(a))$ for $k \geq i$ as in Definition~\ref{def:bad_segments}. Similarly to before, we now define the bad round $s'(a)$ as the smallest round $s'(a) > \tau_i$ such that for the same constant $c_4>0$ as in \eqref{eq:bad_round}
    \begin{equation}\label{eq:bad_round_2}
        \sum_{\substack{(k,j)\colon s_{k,j+1}(a) < s'(a)}} \sqrt{s_{k,j+1}(a) - s_{k,j}(a)} > c_4 \log(T) \sqrt{s'(a) - t_\ell},
    \end{equation}
    where the sum is over all bad segments with $k \geq i$ and $s_{k, j+1}(a) < s'(a)$.

    Importantly, for this definition of $s'(a)$ and with the sum $X(t_\ell, s'(a))$ defined accordingly, the high probability guarantee of Lemma~\ref{lem:replay_scheduled_whp} still holds. 
    This implies that a perfect replay (see Proposition~\ref{prop:perfect_replay}) that eliminates arm $a$ (from the unconfined replay $\base(s, m)$) is scheduled w.h.p.\ before the bad round $s'(a)$ occurs. Let the corresponding event denote $\cE''(t_\ell)$ as in Lemma~\ref{lem:replay_scheduled_whp}.

    The round $M(s, m, a)$ was defined as the last round for which $a$ is retained in $\base(s, m)$ and all of its children. Hence, on event $\cE \cap \cE''(t_\ell)$, we must have $M(s, m, a) < s'(a)$ as otherwise $a$ would have been eliminated from $\base(s, m)$ (or one of its children) before round $M(s, m, a)$, a contradiction. By merit of \eqref{eq:bound_regret_using_segments}, this yields
    \begin{equation*}
        \sum_{\substack{(k,j)\colon s_{k,j+1}(a) < M(s, m, a)}} \sqrt{s_{k,j+1}(a) - s_{k,j}(a)} \leq c_4 \log(T) K \sqrt{M(s, m, a) - t_\ell}
    \end{equation*}
    The regret on the final segment $[s_{k,j} (a), M(s, m, a)]$ can trivially be bounded by $c_3 \log(T) K \sqrt{m}$, as it must be a non-bad segment and $M(s,m,a) - s_{k,j}(a) \leq m$. Finally, in view of \eqref{eq:bound_regret_using_segments}, it follows that
    \begin{align*}
        \sum_{t=s \vee t_\ell^a}^{M(s, m, a)} \delta_t(a_t^*, a) \1_{\{ t \in \badseg(a)\}} 
        & \leq c_5 \log^2(T) K (\sqrt{M(s,m,a)-t_\ell} + \sqrt{m}) \\
        & \leq c_5 \log^2(T) K (\sqrt{s - t_\ell} + 2 \sqrt{m}), 
    \end{align*}
    where the second inequality uses $\sqrt{M(s, m, a) - t_\ell} \leq \sqrt{s-t_\ell} + \sqrt{m}$, since $M(s, m, a) \leq s+m$ and $s \geq t_\ell$.

\end{proof}

\subsubsection{Combining Confined and Unconfined Replays}
We will now conclude the bound on $\termb$. To this end, recall that the replay schedule is chosen according to $B_{s,m} \mid t_\ell \sim\text{Bern}\big({1}/{\sqrt{m (s-t_{\ell})}}\big)$. Then, conditioning on $t_\ell$, we have 
\begin{align*}
    \E \left[ \sum_{s=t_\ell+1}^{t_{\ell+1}} \sum_m \1_{\{B_{s,m} \}}\right] = \E \left[ \sum_{s=t_\ell+1}^{T} \sum_m \E\left[\1_{\{B_{s,m} \}} \mid t_\ell \right] \, \E\left[\1_{\{s < t_{\ell+1}\}} \mid t_\ell \right]  \right] = \E \left[ \sum_{s=t_{\ell+1}}^{t_{\ell+1}-1} \frac{1}{\sqrt{m(s-t_\ell)}}\right]. 
\end{align*}
Moreover, note that we can rewrite a sum over $s\in [t_\ell+1, t_{\ell+1})$ as a double sum over $i \in \phases(t_\ell, t_{\ell+1})$ and $s \in [\tau_i \vee (t_\ell+1), \tau_{i+1}\wedge t_{\ell+1})$. For unconfined replays, we notice that when $\base(s, m)$ is scheduled with $s\in [\tau_i, \tau_{i+1})$, it must hold that $m \geq \tau_{i+1}-s$, as $\base(s, m)$ would otherwise not be unconfined. 

Now, combining Lemma~\ref{lem:confined_replay_bound} and Lemma~\ref{lem:unconfined_replay_bound}, we obtain
\begin{align*}
    & \E \left[ \1_{\{\cE \cap \cE''(t_\ell)\}} \sum_{a=1}^K \sum_{t=t_\ell^a}^{t_{\ell+1}-1} \frac{\delta_t (a_t^*, a)}{|\cA_t|} \1_{\{a \in \cA_t, t \in \badseg(a)\}} \right] \\[1em]
    & \leq \E \left[ \1_{\{\cE \cap \cE''(t_\ell)\}} \sum_{s=t_\ell+1}^{t_{\ell+1}-1} \sum_{m} \1_{\{ B_{s, m} = 1\}} \sum_{t=s \vee t_\ell^a}^{M(s, m, a)} \frac{\delta_t (a_t^*, a)}{|\cA_t|} \1_{\{ t \in \badseg(a)\}} \right] \\[1em]
    & \leq c_2 K \log(K) \log(T) \,\E \left[ \sum_{s=t_\ell}^{t_{\ell+1}-1} \sum_{m} \frac{\sqrt{m}}{\sqrt{m(s-t_\ell)}} \right] 
    \\ & \hspace{4cm} + c_5 K \log(K) \log^2(T) \, \E \left[ \sum_{i\in \phases(t_\ell, t_{\ell+1})} \sum_{s=\tau_i}^{\tau_{i+1}-1} \sum_{m \geq \tau_{i+1} - s} \frac{\sqrt{s-t_\ell} + 2\sqrt{m}}{\sqrt{m (s-t_\ell)}}  \right] \\[1em]
    & \leq c_2 K \log^3(T) \, \E\left[ \sqrt{t_{\ell+1}-t_\ell} \right] + c_5 K \log^3(T) \,  \E\left[\sum_{i\in \phases(t_\ell, t_{\ell+1})} \sum_{s=\tau_{i}}^{\tau_{i+1}-1} \frac{1}{\sqrt{\tau_{i+1}-s}} + 2 \sqrt{t_{\ell+1}-t_\ell} \right] \\[1em]
    & \leq c_7 K \log^3(T)\,  \E \left[ 2 \sum_{i \in \phases(t_\ell, t_{\ell+1})} \sqrt{\tau_{i+1} - \tau_i}  + \sum_{i \in \phases(t_\ell, t_{\ell+1})} \sqrt{\tau_{i+1} - \tau_i} + 4 \sum_{i \in \phases(t_\ell, t_{\ell+1})} \sqrt{\tau_{i+1} - \tau_i} \right] \\[1em]
    & \leq 7c_7 K \log^3(T) \, \E \left[ \sum_{i \in \phases(t_\ell, t_{\ell+1})} \sqrt{\tau_{i+1} - \tau_i} \right]
\end{align*}

We here repeatedly used that $\sum_{k=1}^n 1/{\sqrt{k}} \leq 2 \sqrt{n}$ in the third and fourth inequality. In particular, the fourth inequality holds as $\sqrt{t_{\ell+1}- t_\ell} \leq \sum_{i\in \phases(t_\ell, t_{\ell+1})} \sqrt{\tau_{i+1}-\tau_i}$ and  
\begin{align*}
    \sum_{s=\tau_i}^{\tau_{i+1}-1} \frac{1}{\sqrt{\tau_{i+1}-s}} = \sum_{s=1}^{\tau_{i+1}-\tau_i -1} \frac{1}{\sqrt{s}}\leq \sqrt{\tau_{i+1}- \tau_i}.
\end{align*}
Further note that, as explained before, the denominator $|\cA_t|$ can be seen to account for a factor of $\log(K)$, which we loosely upper bounded by $\log(T)$. Together with \eqref{eq:non_bad_segments_replays}, we then obtain for some constant $c_8 > 0$ the desired bound of 
\begin{align}\label{eq:final_bound_(b)}
    E \left[\sum_{a=1}^K \sum_{t=t_\ell^a}^{t_{\ell+1}-1} \frac{\delta_t (a_t^*, a)}{|\cA_t|} \1_{\{a \in \cA_t\}} \right] 
    & \leq c_8 K \log^3(T) \, \E \left[ \sum_{i \in \phases(t_\ell, t_{\ell+1})} \sqrt{\tau_{i+1} - \tau_i} \right] .
\end{align}

\subsection{Summing Over Episodes}\label{subsection:summing_over_episodes}
In Section~\ref{subsection:bounding(a)} and Section~\ref{subsection:bounding(b)}, we bounded the regret of arms within an episode before and after their elimination, respectively. Combining \eqref{eq:final_bound_(a)} and \eqref{eq:final_bound_(b)}, and summing over episodes, we then obtain 
\begin{equation*}
    \E \left[ \sum_{t=1}^T \frac{\delta_t(a_t^*, a_t) + \delta_t(a_t^*, b_t)}{2} \right] \leq c_9 K \log^3(T) \, \E \left[ \1_{\{\cE\}} \sum_{\ell= 1}^L \sum_{i\in \phases(t_\ell, t_{\ell+1})} \sqrt{\tau_{i+1}-\tau_i} \right] + \frac{1}{T} . 
\end{equation*}
Now, on the concentration event $\cE$, Lemma~\ref{lem:phase_intersecting_episodes} tells us that any phase $[\tau_{i}, \tau_{i+1})$ intersects with at most two episodes. Recall that $\tau_0 \coloneqq 1$ and $\tau_{\Sw +1} \coloneqq T$. It then follows from the above that 
\begin{equation*}
    \E \left[ \sum_{t=1}^T \frac{\delta_t(a_t^*, a_t) + \delta_t(a_t^*, b_t)}{2} \right] \leq 2 c_9 K \log^3(T) \sum_{i=0}^{\Sw} \sqrt{\tau_{i+1} - \tau_i} + \frac{1}{T} . 
\end{equation*}

\section{Missing Details from Section~\ref{sec:sign_switches}}\label{sec:appendix_proof_sign_thm}

\subsection{Significant CW Switches} 
Let us first recall the definition of Significant Condorcet Winner Switches from Section~\ref{sec:nst_measures}. 
Let $\htau_0 \coloneqq 1$ and define $\htau_{i+1}$ recursively as the first round in $[\htau_i, T)$ such that for all arms $a \in [K]$ there exist rounds $\htau_i \leq s_1 < s_2 < \htau_{i+1}$ such that 
\begin{align}\label{eq:sig_shift}
    \sum_{t=s_1}^{s_2} \delta_t(a_t^*, a) \geq  \sqrt{K(s_2-s_1)} . 
\end{align}
Let $\sigS$ denote the number of such {Significant CW Switches} $\htau_1, \dots, \htau_{\sigS}$. The key idea of \cite{Suk22} when developing this notion of non-stationarity (for multi-armed bandits) is that a restart in exploration is only warranted if there are no \emph{safe} arms left to play, i.e.\ there is no arm left that does not suffer regret~\eqref{eq:sig_shift} on some interval $[s_1, s_2]$. For every phase $[\htau_i, \htau_{i+1})$, we denote by $\asafe_i$ the last safe arm in phase $i$, i.e.\ the last arm to satisfy \eqref{eq:sig_shift} in phase $i$. Moreover. we define the sequence of safe arms as $\asafe_t = \asafe_i$ for $t \in [\htau_i, \htau_{i+1})$. 

Significant CW Switches are able to reconcile switch-based non-stationarity measures such as CW Switches $\Sw$ and variation-based non-stationarity measures such as the CW Variation $\Vcw$. More specifically, it naturally holds that $\sigS \leq \Sw$ and Corollary~\ref{cor:total_variation} shows that near-optimal dynamic regret w.r.t.\ $\sigS$ also implies near-optimal dynamic regret w.r.t.\ $\Vcw$.

\subsection{Proof of Theorem~\ref{thm:sign_cw_changes}}

For convenience, we recall the assumptions of Theorem~\ref{thm:sign_cw_changes}. 

\begin{assumptionp}{1}[Strong Stochastic Transitivity] 
    \label{asm:sst}
    Every preference matrix $\bP_t$ satisfies that if $a \succ_t b \succ_t c$, we have $\delta_t(a,c) \geq \delta_t(a, b) \vee  \delta_t(b, c)$.  
\end{assumptionp}

\begin{assumptionp}{2}[Stochastic Triangle Inequality]
    \label{asm:sti}
    Every preference matrix $\bP_t$ satisfies that if $a \succ_t b \succ_t c$, we have $\delta_t(a,c) \leq \delta_t(a, b) +   \delta_t(b, c)$.  
\end{assumptionp}

We see that together Assumption~\ref{asm:sst} and Assumption~\ref{asm:sti} imply a more general type of triangle inequality for any triplet $a, b, c \in [K]$ with $a \succ b$ and $a \succ c$. 
\begin{lem}\label{lem:triangle_inequality}
    Under Assumption~1 and Assumption~2, for any triplet $a, b, c \in [K]$ with $a \succ_t b$ and $a \succ_t c$, it holds that  
    $$\delta_t(a, c) \leq 2\delta_t (a, b) + \delta_t (b, c). $$
\end{lem} 
\begin{proof}
    Suppose that $b \succ_t c$. Then, the claim follows directly from the stochastic triangle inequality, since $\delta_t(a, c) \leq \delta_t(a, b) + \delta_t (b, c)$. 
    Suppose that $c \succ_t b$. Leveraging strong stochastic transitivity of the triplet $a \succ_t c \succ_t b$, we have 
    $$\delta_t(a, b) \geq \delta_t(a, c) \vee \delta_t (c, b).$$
    This implies that $\delta_t(a, c) \leq \delta_t(a, b)$ as well as $\delta_t(c, b) \leq \delta_t (a, b)$. By definition of the gaps, this also yields $\abs{\delta_t(b, c)} \leq \delta_t (a, b)$, since $c \succ_t b$. Consequently, it holds that $\delta_t(a, c) \leq 2 \delta_t(a, b) + \delta_t(b, c)$. 
\end{proof} 

As briefly discussed in Section~\ref{sec:sign_switches}, these assumptions on the preference sequence $\bP_1, \dots, \bP_T$ allow us to decompose the dynamic regret so that we can compare against a temporarily fixed benchmark.

We can w.l.o.g.\ assume that $a_t^* \succ_t a_t$ and $a_t^* \succ_t \asafe_t$. To see that this assumption is valid, note that $a_t^*$ is the Condorcet winner in round $t$ and it is then easy to verify that Lemma~\ref{lem:triangle_inequality} also holds if $a_t^*$ equals one (or both) of $a_t$ and $\asafe_t$. Applying Lemma~\ref{lem:triangle_inequality} to $a_t^*$, $\asafe_t$ and $a_t$, we have 
\begin{align*}
    \delta_t (a_t^*, a_t) \leq 2 \delta_t( a_t^*, \asafe_t) + \delta_t(\asafe_t, a_t). 
\end{align*} 
Recalling equation \eqref{eq:regret_within_eps} from Section~\ref{sec:appendix_proof_thm}, we then get the following decomposition of the dynamic regret within each episode as
\begin{align*}
     \E\left[ \sum_{t=t_\ell}^{t_{\ell+1}-1} \frac{\delta_t(a_t^*, a_t)  + \delta_t(a_t^*, b_t)}{2} \right] \leq 2 \underbrace{\E\left[ \sum_{t=t_\ell}^{t_{\ell+1}-1} \delta_t(a_t^*, \asafe_t)\right]}_{\termone} + \underbrace{\E\left[ \sum_{t=t_\ell}^{t_{\ell+1}-1} \delta_t(\asafe_t, a_t)\right]}_{\termtwo} .
\end{align*}

\subsubsection{Bounding $\termone$} 
We can bound $\termone$ directly using the definition of Significant CW Switches. By definition of $\asafe_i$ as the last safe arm in phase $[\htau_i, \htau_{i+1})$, i.e.\ the last arm to satisfy \eqref{eq:sig_shift} for some interval $[s_1, s_2] \subseteq [\htau_i, \htau_{i+1})$, it holds that 
\begin{align*}
    \sum_{t=\htau_i}^{\htau_{i+1}} \delta_t(a_t^*, \asafe_i) \leq \sqrt{K (\htau_{i+1}- \htau_i)}. 
\end{align*}
We can then sum over all phases $i\in [\sigS]$ to obtain 
\begin{align*}
    \sum_{t=1}^T \delta_t(a_t^*, \asafe_t) \leq \sum_{i=1}^{\sigS} \sqrt{K (\htau_{i+1}- \htau_i)}. 
\end{align*}

\subsubsection{Bounding $\termtwo$} 
As briefly mentioned in the main text, the difficulty in bounding $\sum_{t=t_\ell}^{t_{\ell+1}-1}\delta_t(a_t^*, a_t)$ for Significant CW Switches is that the identity of the Condorcet winner, i.e.\ $a_t^*$, may change several times within each significant phase $i \in [\sigS]$. This makes accurately tracking $\delta_t(a_t^*, a)$ (nearly) impossible even across small intervals and the arguments that we used to prove Theorem~\ref{thm:main_result} fail. 

In contrast, when we consider the relative regret of $a_t$ against the last safe arm $\asafe_t$ (or sequence thereof), this difficulty can be resolved. Considering $\asafe_t$ (instead of $a_t^*$) as a benchmark ensures that within each phase $i \in [\sigS]$ the comparator arm is fixed, since $\asafe_t = \asafe_i$ for all $t \in [\htau_i , \htau_{i+1})$. Hence, the relative regret w.r.t.\ $\asafe_t$ can still be dealt with. In particular, the proof of Theorem~\ref{thm:main_result} from Section~\ref{sec:appendix_proof_thm} can be seen to hold with minor changes when swapping $a_t^*$ for $\asafe_t$ and considering significant phases $\htau_1, \dots, \htau_{\sigS}$.  
For completeness, we reformulate and prove two important lemmas from Section~\ref{sec:appendix_proof_thm} that relied on properties of $a_t^*$ and $\tau_1, \dots, \tau_{\Sw}$. We want to emphasise that we here again rely on Assumption~\ref{asm:sst} and Assumption~\ref{asm:sti}. 

The following lemma shows that the beginning of a new episode implies a Significant CW Switch, i.e.\ every arm suffers at least \eqref{eq:sig_shift} much regret over some interval within the episode. 

\begin{lem}[Lemma~\ref{lem:phase_intersecting_episodes} for $\sigS$]\label{lem:phase_intersecting_episodes_2}
    On event $\cE$, for each episode $[t_\ell, t_{\ell+1})$ with $t_{\ell+1} \leq T$, there exists a Significant CW Switch $\htau_i \in [t_\ell, t_{\ell+1})$. 
\end{lem}

\begin{proof}
    The start of a new episode means that every arm $a\in [K]$ has been eliminated from $\mA$ at some round in $t_\ell^a \in [t_\ell, t_{\ell+1})$. As a result, there must exist an interval $[s_1, s_2] \subseteq [t_\ell, t_\ell^a)$ and some arm $a' \in[K]$ so that the elimination rule \eqref{eq:elim} holds. Using Lemma~\ref{prop:concentration_bound}, we then find that for some constant $c_2 > 0$: 
    \begin{align}\label{eq:large_regret_2}
        \sum_{t=s_1}^{s_2} \E \left[ \hdelta_t(a', a) \mid \cF_{t-1} \right] > c_2 \log (T) K \sqrt{(s_2 -s_1) \vee K^2}. 
    \end{align}
    Note that by construction of $\hdelta_t(a', a)$, we always have $\delta_t(a', a) \geq \E[ \hat{\delta}_t(a', a) \mid \cF_{t-1}]$ since
    \begin{equation}\label{eq:hdelta_leq_delta_2}
        \E[  \hat{\delta}_t(a', a) \mid \cF_{t-1} ] = \begin{cases}
			\delta_t(a', a) & a', a \in \cA_t\\
			-1/2 & \text{otherwise}. 
		\end{cases}
    \end{equation}
    Applying Lemma~\ref{lem:triangle_inequality} to the triplet $(a_t^*, a', a)$, we get that $\delta_t(a_t^*, a) \geq 2 \delta_t(a_t^*, a') + \delta_t(a', a) \geq \delta_t(a', a)$. Thus, \eqref{eq:large_regret_2} tells us that there exists no arm $a \in [K]$ such that for all $[s_1, s_2] \subseteq [t_\ell, t_{\ell+1})$
    \begin{align*}
        \sum_{t=s_1}^{s_2} \delta_t(a_t^*, a) < \sqrt{K(s_2-s_1)}. 
    \end{align*}
    In other words, there is no arm that remains safe to play throughout the episode and there must have been a Significant CW Switch $\htau_i \in [t_\ell, t_{\ell+1})$. 
\end{proof} 
The following lemma ensures that the last safe arm $\asafe_i$ within phase $i$ is not being eliminated before round $\htau_{i+1}$ by any replay $\base(s, m)$ that is scheduled in said phase. 

\begin{lem}[Lemma~\ref{lemma:no_eviction_inside_replay} for $\asafe_t$]\label{lem:no_eviction_inside_replay_2}
    On event $\cE$, no run of $\base(s, m)$ with $s\in [\htau_i, \htau_{i+1})$ ever eliminates arm $\asafe_i$ before round $\htau_{i+1}$. 
\end{lem}

\begin{proof}
Suppose on the contrary that some $\base(s, m)$ with $s \in [\htau_i, \htau_{i+1})$ eliminates arm $\asafe_i$ before round $\htau_{i+1}$. Then, we must have for some arm $a\in[K]$ and interval $[s_1, s_2] \subseteq [s, \htau_{i+1})$ that 
\begin{equation}
    C \log(T) K\sqrt{(s_2-s_1) \vee K^2} < \sum_{t=s_1}^{s_2} \hdelta_t (a, \asafe_i),
\end{equation}
In view of the concentration bound \eqref{eq:concentration_bound}, this implies on event $\cE$ that 
\begin{equation}
    c_2 \log(T) K \sqrt{(s_2-s_1) \vee K^2} < \sum_{t=s_1}^{s_2} \E \left[ \hdelta_t(a, \asafe_i) \mid \cF_{t-1} \right] \leq \sum_{t=s_1}^{s_2} \delta_t(a, \asafe_i), 
\end{equation}
where the last inequality holds by merit of \eqref{eq:hdelta_leq_delta_2}. Now, by the definition of $\asafe_i$ as the last safe arm in phase $i$, it must hold that $\delta_t(a, \asafe_i) < \sqrt{K (s_2- s_1)}$ for all $t \in [\htau_i, \htau_{i+1})$ and all $a \in [K]$. This stands in contradiction to the above which proves Lemma~\ref{lem:no_eviction_inside_replay_2}.  
\end{proof} 

Now, following the same steps as in the proof of Theorem~\ref{thm:main_result} in Section~\ref{sec:appendix_proof_thm}, we obtain for some constant $\tilde c > 0$
\begin{align*}
    \termtwo \leq \tilde c K \log^3 (T) \E \left[\sum_{i\in \phases_{\sigS}(t_\ell, t_{\ell+1})} \sqrt{\htau_{i+1}- \htau_i} \right],     
\end{align*}
where $\htau_{\sigS +1} \coloneqq T$ and $\phases_{\sigS} (t_1, t_2) \coloneqq \{ i \in [\sigS] \colon [\htau_i, \htau_{i+1}) \cap [t_1, t_2) \neq \emptyset \}$. Lastly, in view of the modified Lemma~\ref{lem:phase_intersecting_episodes_2}, it follows that (cf.\ Section~\ref{subsection:summing_over_episodes}) 
 \begin{equation}\label{eq:sig_cws_regret_bound}
    \DR (T) = \E \left[ \sum_{t=1}^T \frac{\delta_t(a_t^*, a_t) + \delta_t(a_t^*, b_t)}{2} \right] \leq \tilde 2 c K \log^3(T) \sum_{i=0}^{\sigS} \sqrt{\htau_{i+1} - \htau_i}. 
\end{equation}
An application of Jensen's inequality shows that $\DR(T) \leq \tO(K \sqrt{\sigS T})$.

\subsection{Proof of Corollary~\ref{cor:total_variation}}
Recall the definition of the Condorcet Winner Variation from Section~\ref{sec:nst_measures}:  
\begin{align*}
    \Vcw \coloneqq \sum_{t=2}^{T} \max_{a\in [K]} \abs{P_t(a_t^*, a) - P_{t-1}(a_t^*, a)} .
\end{align*}
We define the CW Variation over phase $[\htau_{i}, \htau_{i+1})$ as $\Vcw_{[\htau_i, \htau_{i+1})} \coloneqq \sum_{t= \htau_{i}+1}^{\htau_{i+1}} \max_{a\in [K]} \abs{P_t(a_t^*, a) - P_{t-1}(a_t^*, a)}$. Note that in view of the bound in~\eqref{eq:sig_cws_regret_bound}, it suffices to show that $\sum_{i=0}^{\sigS} K \sqrt{\htau_{i+1}-\htau_i} \leq K \sqrt{T} + \Vcw^{1/3} (KT)^{2/3}$.

Consider a phase $[\htau_i, \htau_{i+1})$ with $0 \leq i < \sigS$. By definition of Significant CW Switches, every arm $a \in [K]$ must satisfy on some interval $[s_1, s_2] \subseteq[\htau_{i}, \htau_{i+1})$ that
\begin{align*}
    \sum_{t=s_1}^{s_2} \delta_t (a_t^*, a) \geq \sqrt{K (s_2-s_1)}. 
\end{align*}
In particular, this is also the case for the Condorcet winner $a_{\htau_{i+1}}^*$ in round $\htau_{i+1}$. Then, since $\sqrt{s_2-s_1} > \sum_{t=s_1}^{s_2} \frac{1}{\htau_{i+1}- \htau_i}$, there exists a round $t \in [s_1, s_2]$ such that $\delta_t(a_t^*, a_{\htau_{i+1}}^*) \geq \sqrt{\frac{K}{\htau_{i+1}-\htau_i}}$. We then have 
\begin{align*}
    \sqrt{\frac{K}{\htau_{i+1}-\htau_i}} & \leq \delta_t (a_t^*, a_{\htau_{i+1}}^*) \\
    & \leq \delta_t (a_t^*, a_{\htau_{i+1}}^*) + \delta_{\htau_{i+1}} (a_{\htau_{i+1}}^*, a_t^*) \\[0.5em]
    & \leq \delta_t (a_t^*, a_{\htau_{i+1}}^*) - \delta_{\htau_{i+1}} (a_t^*, a_{\htau_{i+1}}^*) \\[0.5em]
    & \leq \abs{\delta_t (a_t^*, a_{\htau_{i+1}}^*) - \delta_{\htau_{i+1}} (a_{t}^*, a_{\htau_{i+1}}^*)} \\[0.5em]
    & =  \abs{P_t (a_t^*, a_{\htau_{i+1}}^*) - P_{\htau_{i+1}} (a_t^*, a_{\htau_{i+1}}^*)} \\
    & \leq \sum_{s=t+1}^{\htau_{i+1}} \max_{a \in [K]} \abs{P_t(a_t^*, a) - P_{t-1} (a_t^*, a)} \leq \Vcw_{[\htau_i, \htau_{i+1}]},
\end{align*}
where we used that $\delta_{\htau_{i+1}}(a_{\htau_{i+1}}^*, a_t^*) \geq 0$ and $\delta_{\htau_{i+1}}(a_{\htau_{i+1}}^*, a_t^*) = - \delta_t (a_t^*, a_{\htau_{i+1}}^*)$ in the second and third inequality, respectively.  
We can now apply Hölder's inequality to obtain
\begin{align*}
    \sum_{i=0}^{\sigS} K \sqrt{\htau_{i+1}-\htau_i} & \leq K \sqrt{T} +  \sum_{i=0}^{\sigS - 1}   K \sqrt{\htau_{i+1}-\htau_i} \\ 
    & \leq K \sqrt{T} + \left( \sum_{i=0}^{\sigS} \sqrt{\frac{K}{\htau_{i+1}-\htau_i}} \right)^{1/3} \left( \sum_{i=0}^{\sigS} K^{5/4} (\htau_{i+1}-\htau_i) \right)^{2/3} \\ 
    & \leq K \sqrt{T} + \left( \sum_{i=0}^{\sigS} \Vcw_{[\htau_i, \htau_{i+1})} \right)^{1/3} K^{5/6} \, T^{2/3} \\[0.5em] 
    & = K \sqrt{T} + \Vcw^{1/3}  K^{5/6} \, T^{2/3}.
\end{align*}
The above dependence on $K$ can be improved to $K^{4/9}$ (which is even smaller than the $K^{2/3}$ dependence in Corollary~\ref{cor:total_variation}) by modifying the definition of Significant CW Switches so that $\htau_{i+1}$ is the first round in $[\htau_i, T)$ such that for all arms $a \in [K]$ there exist rounds $\htau_i \leq s_1 < s_2 < \htau_{i+1}$ with 
\begin{align*}
    \sum_{t=s_1}^{s_2} \delta_t(a_t^*, a) \geq K \sqrt{s_2 - s_1}. 
\end{align*}
It is straightforward to verify that Theorem~\ref{thm:sign_cw_changes} also holds true for this definition of Significant CW Switches. 


\section{More Related Work}
\label{app:rel_works}

Related to the non-stationary dueling bandit problem studied in this paper are adversarial dueling bandits~\cite{Ailon+14, Adv_DB, ADB, Sui+17}. Here, \cite{Ailon+14} was the first to study the dueling bandit problem in an adversarial setup and introduced a popular sparring idea, which has been picked up by many follow-up works~\cite{Adv_DB, dudik+15, ADB, SahaNDB}.  
The settings in \cite{Ailon+14} and \cite{Adv_DB} are restricted to utility-based preference models, where each arm is has an associated utility in each round. This entails a complete ordering over the arms in each round, which only covers a small subclass of $[K]\times [K]$ preference matrices. Moreover, \cite{Adv_DB} assume that the feedback includes not only the winner but also the difference in the utilities between the winning and losing arm, which is more similar to MAB feedback and than the $0/1$ one bit preference feedback considered by us. \cite{ADB} consider the dueling bandit setup for general adversarial preferences, but they measure performance in terms of (static) regret w.r.t.\ \emph{Borda-scores}. 
This measure of regret is very different from our preference-based regret objective. 
In general, the adversarial dueling bandit problem aims to minimize \emph{static regret} w.r.t.\ some fixed benchmark $a^*$, whereas we study \emph{dynamic regret} w.r.t.\ a time-varying benchmark $a_t^*$. As discussed in Section~\ref{sec:prob}, static regret can be an undesirable measure of performance when no single fixed arm represents a reasonably good benchmark over all rounds (see Example~\ref{eg:dyn_reg}).   

Another somewhat related line of work considers the sleeping dueling bandit problem, where the action space is non-stationary (as opposed to the preference sequence). The objective here is to be competitive w.r.t.\ the best active arm at each round. \cite{SDB} studies the setup for adversarial sleeping but assumes a fixed preference matrix across all rounds.



\end{document}